\documentclass[runningheads]{llncs}
\usepackage[dvipsnames,table]{xcolor}
\usepackage{graphicx}
\usepackage{comment}
\usepackage{amsmath,amssymb} % define this before the line numbering.
\usepackage{color}

\usepackage{epsfig}
\usepackage{breakcites}

\usepackage{array,multirow}

\usepackage{xcolor}
\usepackage{xspace}

\usepackage{algorithmic}

\usepackage[utf8]{inputenc}
\usepackage[english]{babel}

\usepackage{amsthm}
\usepackage{amsmath} 

\usepackage{collcell}

\usepackage{subfig}

\usepackage{booktabs}
\usepackage{siunitx}
\usepackage{algorithm}

\usepackage{xr}
\usepackage[toc,page]{appendix}

\begin{document}
% \renewcommand\thelinenumber{\color[rgb]{0.2,0.5,0.8}\normalfont\sffamily\scriptsize\arabic{linenumber}\color[rgb]{0,0,0}}
% \renewcommand\makeLineNumber {\hss\thelinenumber\ \hspace{6mm} \rlap{\hskip\textwidth\ \hspace{6.5mm}\thelinenumber}}
% \linenumbers
\pagestyle{headings}
\mainmatter
\def\ECCVSubNumber{3047}  % Insert your submission number here

\title{Multitask Learning Strengthens\\Adversarial Robustness}

%\title{Adversarial Robustness is Linked \\to Number of Tasks}

%Learning Versatile Models Improves Adversarial Robustness} % Replace with your title

% INITIAL SUBMISSION 
\begin{comment}
\titlerunning{ECCV-20 submission ID \ECCVSubNumber} 
\authorrunning{ECCV-20 submission ID \ECCVSubNumber} 
\author{Anonymous ECCV submission}
\institute{Paper ID \ECCVSubNumber}
\end{comment}
%******************

% CAMERA READY SUBMISSION
% \begin{comment}

% If the paper title is too long for the running head, you can set
% an abbreviated paper title here
%
\newcommand*\samethanks[1][\value{footnote}]{\footnotemark[#1]}

\author{Chengzhi Mao \and
Amogh Gupta\thanks{Equal Contribution. Order is alphabetical.} \and
Vikram Nitin\samethanks \and \\
Baishakhi Ray  \and 
Shuran Song \and
Junfeng Yang  \and
Carl Vondrick }
%\inst{1}\orcidID{} is removed

%
\authorrunning{C. Mao et al.}
% First names are abbreviated in the running head.
% If there are more than two authors, 'et al.' is used.
%
\institute{Columbia University, New York, NY, USA \\ \email{mcz,rayb,shurans,junfeng,vondrick@cs.columbia.edu,
ag4202,vikram.nitin@columbia.edu}}

% \end{comment}
%******************
\maketitle

\begin{abstract}
Although deep networks achieve strong accuracy on a range of computer vision benchmarks, they remain vulnerable to adversarial attacks, where imperceptible input perturbations fool the network. We present both theoretical and empirical analyses that connect the adversarial robustness of a model to the number of tasks that it is trained on. Experiments on two datasets show that attack difficulty increases as the number of target tasks increase. Moreover, our results suggest that when models are trained on multiple tasks at once, they become more robust to adversarial attacks on individual tasks. While adversarial defense remains an open challenge, our results suggest that deep networks are vulnerable partly because they are trained on too few tasks.

\end{abstract}

%auto-ignore
\def\Blue{\color{blue}}
\def\Purple{\color{purple}}
$\newcommand{\Cov}{\mathrm{Cov}}$

\def\A{{\bf A}}
\def\a{{\bf a}}
\def\B{{\bf B}}
\def\b{{\bf b}}
\def\C{{\bf C}}
\def\c{{\bf c}}
\def\D{{\bf D}}
\def\d{{\bf d}}
\def\E{{\bf E}}
\def\e{{\bf e}}
\def\f{{\bf f}}
\def\F{{\bf F}}
\def\K{{\bf K}}
\def\k{{\bf k}}
\def\L{{\bf L}}
\def\H{{\bf H}}
\def\h{{\bf h}}
\def\G{{\bf G}}
\def\g{{\bf g}}
\def\I{{\bf I}}
\def\R{{\bf R}}
\def\X{{\bf X}}
\def\Y{{\bf Y}}
\def\OO{{\bf O}}
\def\oo{{\bf o}}
\def\P{{\bf P}}
\def\Q{{\bf Q}}
\def\r{{\bf r}}
\def\s{{\bf s}}
\def\S{{\bf S}}
\def\t{{\bf t}}
\def\T{{\bf T}}
\def\x{{\bf x}}
\def\y{{\bf y}}
\def\z{{\bf z}}
\def\Z{{\bf Z}}
\def\M{{\bf M}}
\def\m{{\bf m}}
\def\n{{\bf n}}
\def\U{{\bf U}}
\def\u{{\bf u}}
\def\V{{\bf V}}
\def\v{{\bf v}}
\def\W{{\bf W}}
\def\w{{\bf w}}
\def\0{{\bf 0}}
\def\1{{\bf 1}}
\def\N{{\bf N}}

\def\AM{{\mathcal A}}
\def\EM{{\mathcal E}}
\def\FM{{\mathcal F}}
\def\TM{{\mathcal T}}
\def\UM{{\mathcal U}}
\def\XM{{\mathcal X}}
\def\YM{{\mathcal Y}}
\def\NM{{\mathcal N}}
\def\OM{{\mathcal O}}
\def\IM{{\mathcal I}}
\def\GM{{\mathcal G}}
\def\PM{{\mathcal P}}
\def\LM{{\mathcal L}}
\def\MM{{\mathcal M}}
\def\DM{{\mathcal D}}
\def\SM{{\mathcal S}}
\def\RB{{\mathbb R}}
\def\EB{{\mathbb E}}

\def\tx{\tilde{\bf x}}
\def\ty{\tilde{\bf y}}
\def\tz{\tilde{\bf z}}
\def\hd{\hat{d}}
\def\HD{\hat{\bf D}}
\def\hx{\hat{\bf x}}
\def\hR{\hat{R}}

\def\Ome{\mbox{\boldmath$\omega$\unboldmath}}
\def\bet{\mbox{\boldmath$\beta$\unboldmath}}
\def\et{\mbox{\boldmath$\eta$\unboldmath}}
\def\ep{\mbox{\boldmath$\epsilon$\unboldmath}}
\def\ph{\mbox{\boldmath$\phi$\unboldmath}}
\def\Pii{\mbox{\boldmath$\Pi$\unboldmath}}
\def\pii{\mbox{\boldmath$\pi$\unboldmath}}
\def\Ph{\mbox{\boldmath$\Phi$\unboldmath}}
\def\Ps{\mbox{\boldmath$\Psi$\unboldmath}}
\def\pss{\mbox{\boldmath$\psi$\unboldmath}}
\def\tha{\mbox{\boldmath$\theta$\unboldmath}}
\def\Tha{\mbox{\boldmath$\Theta$\unboldmath}}
\def\muu{\mbox{\boldmath$\mu$\unboldmath}}
\def\Si{\mbox{\boldmath$\Sigma$\unboldmath}}
\def\Gam{\mbox{\boldmath$\Gamma$\unboldmath}}
\def\gamm{\mbox{\boldmath$\gamma$\unboldmath}}
\def\Lam{\mbox{\boldmath$\Lambda$\unboldmath}}
\def\De{\mbox{\boldmath$\Delta$\unboldmath}}
\def\vps{\mbox{\boldmath$\varepsilon$\unboldmath}}
\def\Up{\mbox{\boldmath$\Upsilon$\unboldmath}}
\def\Lap{\mbox{\boldmath$\LM$\unboldmath}}
\newcommand{\ti}[1]{\tilde{#1}}

\def\tr{\mathrm{tr}}
\def\etr{\mathrm{etr}}
\def\etal{{\em et al.\/}\,}
\newcommand{\indep}{{\;\bot\!\!\!\!\!\!\bot\;}}
\def\argmax{\mathop{\rm argmax}}
\def\argmin{\mathop{\rm argmin}}
\def\vec{\text{vec}}
\def\cov{\text{cov}}
\def\dg{\text{diag}}

\newcommand{\cityscape}{Cityscapes\xspace}
\newcommand{\taskonomy}{Taskonomy\xspace}
\newcommand{\multitask}{multitask\xspace}
\newcommand{\Multitask}{Multitask\xspace}

\newcommand{\cz}[1]{{\scriptsize \todo{Chengzi:  \blue{ #1}}}}
\newcommand{\jf}[1]{{\scriptsize \todo{Junfeng:  \blue{ #1}}}}
\newcommand{\bray}[1]{{\scriptsize \todo{Bray:  {\Red #1}}}}
\newcommand{\amogh}[1]{{\scriptsize \todo{Amogh:  {\Red #1}}}}

\section{Introduction} \label{sec:intro}

\begin{figure}[t]
	\centering
	\includegraphics[width=1\textwidth]{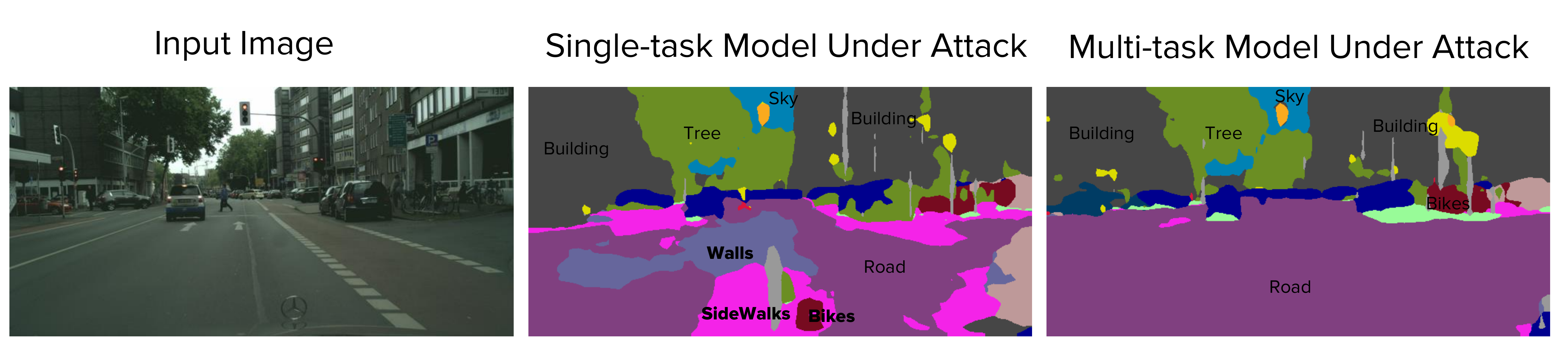}
	\caption{We find that \multitask models are more robust against adversarial attacks. Training a model to solve multiple tasks improves the robustness when one task is attacked. The middle and right column show predictions for single-task and \multitask models when one task is adversarially attacked.}
	\label{fig:intro}
\end{figure}

Deep networks obtain high performance in many computer vision tasks \cite{ResNet, Yu2017,FCN,depth}, yet they remain brittle to adversarial examples. A large body of work has demonstrated that images with human-imperceptible noise \cite{madry,cw,mim,JSMA} can be crafted to cause the model to mispredict. This pervasiveness of adversarial examples exposes key limitations of deep networks, and hampers their deployment in safety-critical applications, such as autonomous driving.

A growing body of research has been dedicated to answering what causes deep networks to be fragile to adversarial examples and how to improve robustness \cite{first-order,notbugbutfeature,TLA,madry,ALP,feature_denoising,ensemble,DefenseGAN,parseval,intriguing}. The investigations center around two factors: the training data and the optimization procedure. For instance, more training data -- both labeled and unlabeled -- improves robustness \cite{moredata,unlabeled}. It has been theoretically shown that decreasing the input dimensionality of data improves robustness \cite{first-order}. Adversarial training \cite{madry} improves robustness by dynamically augmenting the training data using generated adversarial examples. Similarly, optimization procedures that regularize the learning specifically with robustness losses have been proposed \cite{deepdefense,parseval}. This body of work suggests that the fragility of deep networks may stem from the training data and optimization procedure.

In this paper, we pursue a new line of investigation: how learning on multiple tasks affects adversarial robustness.
While previous work shows that \multitask learning can improve the performance of specific tasks \cite{multi-task-caruana, multi-task-onomy}, we show that it increases robustness too. See Figure \ref{fig:intro}. Unlike prior work that trades off performance between natural and adversarial examples \cite{robustness_vs_accuracy}, our work improves adversarial robustness while also maintaining performance on natural examples. 

Using the first order vulnerability of neural networks \cite{first-order}, we theoretically show that increasing output dimensionality -- treating each output dimension as an individual task -- improves the robustness of the entire model. Perturbations needed to attack multiple output dimensions cancel each other out. We formally quantify and upper bound how much robustness a \multitask model gains against a \multitask attack with increasing output dimensionality.

We further empirically show that \multitask learning improves the model robustness for two classes of attack: both when a single task is attacked or several tasks are simultaneously attacked.
We experiment with up to 11 vision tasks on two natural image datasets, \cityscape~\cite{Cityscapes} and \taskonomy~\cite{taskonomy}.
When all tasks are under attack, \multitask learning increases segmentation robustness by up to 7 points and reduces the error of other tasks up to 60\% over baselines.
We compare the robustness of a model trained for a main task with and without an auxiliary task. Results show that, when the main task is under attack, \multitask learning improves segmentation overlap by up to 6 points and reduces the error of the other tasks by up to 23\%. Moreover, \multitask training is a complementary defense to adversarial training, and it improves both the clean and adversarial performance of the state-of-the-art, adversarially trained, single-task  models. Code is available at \url{https://github.com/columbia/MTRobust}. 

Overall, our experiments show that \multitask learning improves adversarial robustness while maintaining most of the the state-of-the-art single-task model performance. While defending against adversarial attacks remains an open problem, our results suggest that current deep networks are vulnerable partly because they are trained for too few tasks.

\section{Related Work}

We briefly review related work in \multitask learning and adversarial attacks. 

\textbf{\Multitask Learning:} 
\Multitask learning \cite{multi-task-caruana,regularized_MTL,MTL_self_sup,TaskGroup_MTL, multi-task-onomy} aims to solve several tasks at once, and has been used to learn better models for semantic segmentation~\cite{att-e2e-mtl}, depth estimation~\cite{upnet}, key-point prediction~\cite{keypoints}, and object detection~\cite{SSD}. It is hypothesized that \multitask learning improves the performance of select tasks by introducing a knowledge-based inductive bias \cite{multi-task-caruana}. However, multi-objective functions are hard to optimize, where researchers design architectures~\cite{onemodelforall, ubernet,task-expert, cross-stitch, liu2019multitask} and optimization procedures~\cite{gradnormMTL, mgda, sener2018multitask, gradientsurgery} for learning better \multitask models. Our work complements this body of work by linking  \multitask learning to adversarial robustness.

\textbf{Adversarial Attacks:}
Current adversarial attacks manipulate the input \cite{intriguing,mim,costales2020live, RobustSeg, houdini, shen2019advspade, DAG, Metzen_2017} to fool target models. While attacking single output models~\cite{corr/GoodfellowSS14,BIM} is straightforward, Arnab et. al.~\cite{RobustSeg} empirically shows the inherent hardness of attacking segmentation models with dense output. Theoretical insight of this robustness gain, however, is missing in the literature. While past theoretical work showed the hardness of multi-objective optimization \cite{hardmultiobject, Multitask-influence}, we leverage this motivation and prove that \multitask models are robust when tasks are simultaneously attacked. Our work contributes both theoretical and empirical insights on adversarial attacks through the lens of \multitask learning.

\textbf{Adversarial Robustness:}
Adversarial training improves models' robustness against attacks, where the training data is augmented using adversarial samples \cite{corr/GoodfellowSS14, madry}. In combination with adversarial training, later works \cite{ALP,TLA,TRADES, feature_denoising} achieve improved robustness by regularizing the feature representations with additional loss, which can be viewed as adding additional tasks. Despite the improvement of robustness, adversarially trained models lose significant accuracy on clean (unperturbed) examples \cite{madry,TRADES, robustness_vs_accuracy}. Moreover, generating adversarial samples slows down training several-fold, which makes it hard to scale adversarial training to large datasets.

Past work revealed that model robustness is strongly connected to the gradient of the input, where models' robustness is improved by regularizing the gradient norm \cite{regularize-input-grad,deepdefense,parseval}. Parseval \cite{parseval} regularizes the Lipschitz constant---the maximum norm of gradient---of the neural network to produce a robust classifier, but it fails in the presence of batch-normalization layers. \cite{regularize-input-grad} decreases the input gradients norm. These methods can improve the model's robustness without compromising clean accuracy. Simon-Gabriel et al. \cite{first-order} conducted a theoretical analysis of the vulnerability of neural network classifiers, and connected gradient norm and adversarial robustness.
Our method enhances robustness by training a \multitask model, which complements both adversarial training \cite{madry, TRADES} and existing regularization methods \cite{PED,regularize-input-grad,deepdefense,parseval}.

\section{Adversarial Setting} \label{sec:adv_setting}

The goal of an adversary is to ``fool'' the target model by adding human-imperceptible perturbations to its input. We focus on untargeted attacks, which are harder to defend against than targeted attacks \cite{EvalALP}.
We classify adversarial attacks for a \multitask prediction model into two categories: adversarial attacks that fool more than one task at once (\multitask attacks), and adversarial attacks that fool a specific task (single-task attacks).

\subsection{\Multitask Learning Objective}

\noindent
\textbf{Notations.} Let $\x$ denote an input example, and $\y_c$ denote the corresponding ground-truth label for task $c$. In this work, we focus on \multitask learning with shared parameters \cite{ubernet,multi-task-onomy, luong2015multitaskshare, Liu_2019_MTLshare, lee2019MTLgeneralization}, where all the tasks share the same ``backbone network'' $F(\cdot)$ as a feature extractor with task-specific decoder networks $D_c(\cdot)$. 
The task-specific loss is formulated as:
\begin{equation}
\mathcal{L}_c(\x, \y_c) = \ell(D_c(F(\x)), \y_c),
\end{equation}
\noindent where $\ell$ is any appropriate loss function. For simplicity, we denote $(\y_1, ..., \y_M)$ as $\overline{\y}$, where $M$ is the number of tasks. The total loss for \multitask learning is a weighted sum of all the individual losses:
\begin{equation}\label{sec:joint}
\mathcal{L}_{all}(\x, \overline{\y}) = \sum_{c=1}^{M}{\lambda_c \mathcal{L}_c(\x, \y_c)}
\end{equation}
\noindent For the simplicity of theoretical analysis, we set $\lambda_c=\frac{1}{M}$ for all $c=1,..., M$, such that $\sum_{c=1}^M \lambda_c = 1$. In our experiments on real-world datasets, we will adjust the $\lambda_c$ accordingly, following standard practice \cite{multi-task-onomy, lee2019MTLgeneralization}.

\subsection{Adversarial \Multitask Attack Objective}
The goal of a \multitask attack is to change multiple output predictions together. For example, to fool an autonomous driving model, the attacker may need to deceive both the object classification and depth estimation tasks.  Moreover, if we regard each output pixel of a semantic segmentation task as an individual task, adversarial attacks on segmentation models need to flip multiple output pixels, so we consider them as \multitask attacks. We also consider other dense output tasks as a variant of \multitask, such as depth estimation, keypoints estimation, and texture prediction.

In general, given an input example $\x$, the objective function for \multitask attacks against models with multiple outputs is the following:
\begin{equation}
\begin{split}
\argmax_{\x_{adv}} \mathcal{L}_{all}(\x_{adv}, \overline{\y}) \quad 
\text{s.t.} \quad ||\x_{adv}-\x||_p \leq r
\end{split}
\end{equation}
where the attacker aims to maximize the joint loss function by finding small perturbations within a $p$-norm bounded distance $r$ of the input example. Intuitively, a \multitask attack is not easy to perform because the attacker needs to optimize the perturbation to fool each individual task simultaneously. The robustness of the overall model can be a useful property - for instance, consider an autonomous-driving model trained for both classification and depth estimation. If either of the two tasks is attacked, the other can still be relied on to prevent accidents.

\subsection{Adversarial Single-Task Attack Objective}
In contrast to a \multitask attack, a single-task attack focuses on a selected target task. Compared with attacking all tasks at once, this type of attack is more effective for the target task, since the perturbation can be designed solely for this task without being limited by other considerations. It is another realistic type of attack because some tasks are more important than the others for the attacker. For example, if the attacker successfully subverts the color prediction for a traffic light, the attacker may cause an accident even if the other tasks predict correctly.
The objective function for single-task attack is formulated as:
\begin{equation}
\begin{split}
\argmax_{\x_{adv}} \mathcal{L}_c(\x_{adv}, \y_c),    \text{s.t.} ||\x_{adv} - \x||_p \leq r
\end{split}
\end{equation}{}
\noindent For any given task, this single-task attack is more effective than jointly attacking the other tasks. We will empirically demonstrate that \multitask learning also improves model robustness against this type of attack in Section \ref{sec:eval}.

\section{Theoretical Analysis}% of Adversarial Vulnerability for Multi-task Models}
\label{sec:theory}
%  then we will support the theoretical with empirical result on real-world cityscapes datasets.

We present theoretical insights into the robustness of \multitask models. A prevalent formulation of \multitask learning work uses shared backbone network with task-specific branches \cite{luong2015multitaskshare, Liu_2019_MTLshare, lee2019MTLgeneralization}. We denote the \multitask predictor as $F$ and each individual task predictor as $F_c$. Prior work \cite{first-order} showed that the norm of gradients captures the vulnerability of the model. We thus measure the \multitask models' vulnerability with the same metric. Since we are working with deep networks, we assume all the functions here are differentiable.  Details of all proofs are in the supplementary material.

\begin{definition}
Given  classifier $F$, input $\x$, output target $\y$, and loss $\mathcal{L}(\x, \y) = \ell(F(\x), \y)$, the feasible adversarial examples lie in a $p$-norm bounded ball with radius $r$, $B(\x, r) := \{\x_{adv}, ||\x_{adv}-\x||_p<r\}$. Then adversarial vulnerability of a classifier over the whole dataset is 
\[\mathbb{E}_{\x} [\Delta \mathcal{L}(\x, \y, r)] = \mathbb{E}_{\x} [\max_{||\delta||_p < r}{|\mathcal{L}(\x,\y) - \mathcal{L}(\x+\delta,\y)|}]\]
% we call a successful $\ell_p$-norm bounded attack by $\forall c \in \{1...M\} \ \  F_c(x+\delta) \neq y_c$, s.t. $||\delta||_p < \epsilon$
\end{definition}

% \noindent Assume the given classifier $F$ predict correct under normal input examples, i.e., $F_c(\x) = y_c$, we rewrite the above definition as:

% \begin{equation}
    % \forall i \in {1, ..., M}, F_c(\x + \delta) \neq F_c(\x)
% \end{equation}

\noindent $\Delta \mathcal{L}$ captures the maximum change of the output loss from arbitrary input change $\delta$ inside the $p$-norm ball. Intuitively, a robust model should have smaller change of the loss given any perturbation of the input. Given the adversarial noise is imperceptible, i.e., $r \rightarrow 0$, we can approximate $\Delta \mathcal{L}$ with a first-order Taylor expansion \cite{first-order}. 

\begin{lemma}
For a given neural network $F$ that predicts multiple tasks, the adversarial vulnerability is
\[\mathbb{E}_{\x} [\Delta \mathcal{L}(\x, \y, r)] \approx  \mathbb{E}_{\x}\left[||\partial_{\x} \mathcal{L}_{all}(\x, \overline{\y})||_q \right] \cdot ||\delta||_p \propto \mathbb{E}_{\x}\left[||\partial_{\x} \mathcal{L}_{all}(\x, \overline{\y})||_q \right] \]
\end{lemma}

\noindent where $q$ is the dual norm of $p$, which satisfies $\frac{1}{p} + \frac{1}{q} = 1$ and $1 \leq p \leq \infty$. Without loss of generality, let $p=2$ and $q=2$. Note that from equation \ref{sec:joint}, we get $\mathcal{L}_{all}(\x, \overline{\y}) = \sum_{c=1}^{M}{\frac{1}{M} \mathcal{L}_c(\x, \y_c)}$. Thus we get the following equation:
\begin{equation}\label{equ:lamma_remark}
\begin{split}
    \partial_{\x} \mathcal{L}_{all}(\x, \overline{\y}) = \partial_{\x} \sum_{c=1}^{M}{\frac{1}{M} \mathcal{L}_c(\x, \y_c) } 
     = \frac{1}{M} \sum_{c=1}^{M}{\partial_{\x}  \mathcal{L}_c(\x, \y_c) }
\end{split}
\end{equation}
We denote the gradient for task $c$ as $\r_c$, i.e., $\r_c = \partial_{\x}\mathcal{L}_c(\x, \y_c)$. We propose the following theory for robustness of different numbers of randomly selected tasks.
\begin{theorem}\label{th:2}
\textbf{(Adversarial Vulnerability of Model for Multiple Correlated Tasks)}
If the selected output tasks are correlated with each other such that the covariance between the gradient of task $i$ and task $j$ is $\mathrm{Cov}(\r_i, \r_j)$, and the gradient for each task is i.i.d. with zero mean (because the model is converged), then adversarial vulnerability of the given model is proportional to 
\[\sqrt{\frac{1 + \frac{2}{M} \sum_{i=1}^{M} \sum_{j=1}^{i-1} \frac{\mathrm{Cov}(\r_i, \r_j)}{\mathrm{Cov}(\r_i, \r_i)}}{M}}\]
where $M$ is the number of output tasks selected.
\end{theorem}
The idea is that when we select more tasks as attack targets, the gradients for each of the individual tasks on average cancels out with each other. We define the joint gradient vector $\R$ as follows:
\begin{equation*}
    \R = \partial_{\x} \mathcal{L}_{all}(\x, \overline{\y}) = \frac{1}{M} \sum_{c=1}^{M}{\partial_{\x}\mathcal{L}_c(\x, \y_c)}
\end{equation*}
The joint gradient is the sum of gradients from each individual task. We then obtain the expectation of the $L_2$ norm of the joint gradient:
\begin{align*}
    % \hspace{-60}
    &\mathbb{E}(\Vert \R \rVert_2^2) = \mathbb{E}\left[\lVert\frac{1}{M} \sum_{i=1}^M{\r_i}\rVert_2^2\right]
    = \frac{1}{M^2}\mathbb{E}\left[\sum_{i=1}^M{\lVert \r_i \rVert^2} + 2 \sum_{i=1}^M{\sum_{j=1}^i{\r_i \r_j}}\right] \\
    % &= \frac{1}{M^2}(\sum_{i=1}^M\mathbb{E}{||\r_i||^2} + 2 \sum_{i=1}^M{\sum_{j=1}^i\mathbb{E}({\r_i \r_j})}) \\
    &= \frac{1}{M^2}\left(\sum_{i=1}^M\mathbb{E}[\mathrm{Cov}(\r_i, \r_i)] + 2 \sum_{i=1}^M{\sum_{j=1}^i\mathbb{E}[\mathrm{Cov}(\r_i, \r_j)]}\right) 
\end{align*}
\noindent The last equation holds due to the $\0$ mean assumption of the gradient. For further details of the proof, please see the supplementary material.

% \begin{remark}
% If the correlation of all outputs is 1, i.e., all tasks are the same, then the vulnerability of the model is \emph{constant} regardless of the increase in the number of tasks.
% \end{remark}{}

\begin{corollary}\label{th:1}
\textbf{(Adversarial Vulnerability of Model for Multiple Independent Tasks)}
 If the output tasks selected are independent of each other, and the gradient for each task is i.i.d. with zero mean, then the adversarial vulnerability of given model is proportional to $\frac{1}{\sqrt{M}}$, where $M$ is the number of independent output tasks selected. 
\end{corollary}
Based on the independence assumption, all covariances becomes zero. Thus  Theorem \ref{th:1} can be simplified as:
\begin{equation}
\begin{split}
    \mathbb{E}[\lVert \partial_{\x} \mathcal{L}_{all}(\x, \overline{\y})\rVert_2^2] = \mathbb{E}(\lVert \R \rVert_2^2) 
    = \frac{1}{M}\mathbb{E}{\lVert \r_i \rVert^2} = \frac{\sigma^2}{M} \propto \frac{1}{M}
    \end{split}{}
\end{equation}{}
\begin{remark}
By increasing the number of output tasks $M$, the first order vulnerability~\cite{first-order} of network decreases. In the ideal case, if the model has an infinite number of uncorrelated tasks, then it is impossible to find an adversarial examples that fools all the tasks.
\end{remark}
\begin{remark}
Theorem 1 studies the robustness for multiple \textbf{correlated} tasks, which is true for most computer vision tasks. The independent tasks assumption in Corollary 1 is a simplified, idealistic instance of Theorem 1 that upper-bounds the robustness of models under \multitask attacks. Together, Theorem 1 and Corollary 1.1 demonstrate that unless the tasks are 100\% correlated (the same task), multiple tasks together are more robust than each individual one. 
\end{remark}
Our theoretical analysis shows that more outputs, especially if they are less correlated, improve the model's robustness against \multitask attacks. Past work shows that segmentation is inherently robust \cite{RobustSeg, houdini} compared to classification. Our analysis provides a formal explanation to this inherent robustness because a segmentation model can be viewed as a \multitask model (one task per pixel).

% As we can see from the theory, if , then. For example, if we simply duplicate the output of the model many times, the vulnerability of the model will not change, since the duplicated outputs are fully correlated with the existing ones.

\section{Experiments} \label{sec:eval}

We validate our analysis with empirical results on the \cityscape and the \taskonomy datasets. We evaluate the robustness of \multitask models against two types of attack: \multitask attack (Section \ref{sec:multi-attack}) and single-task attacks (Section \ref{sec:single-attack}).  We also conduct \multitask learning experiments on adversarial training and show that they are complementary (Section \ref{sec:adv-train}).
% Move to supplementary : All the experiments were done on a Tesla V100 GPU with 16GB memory.

\subsection{Datasets}
\textbf{\cityscape.} The \cityscape dataset \cite{Cityscapes} consists of images of urban driving scenes. We study three tasks: semantic segmentation, depth estimation, and image reconstruction. We use the full resolution ($2048 \times 1024$) for analyzing pre-trained state-of-the-art models. We resize the image to $680 \times 340$ to train our single task (baseline) and \multitask models.\footnote{We use the same dimension for baselines and ours during comparison because input dimension impacts robustness \cite{first-order}.}

\begin{figure}[t!]
	\centering
	\includegraphics[width=1\textwidth]{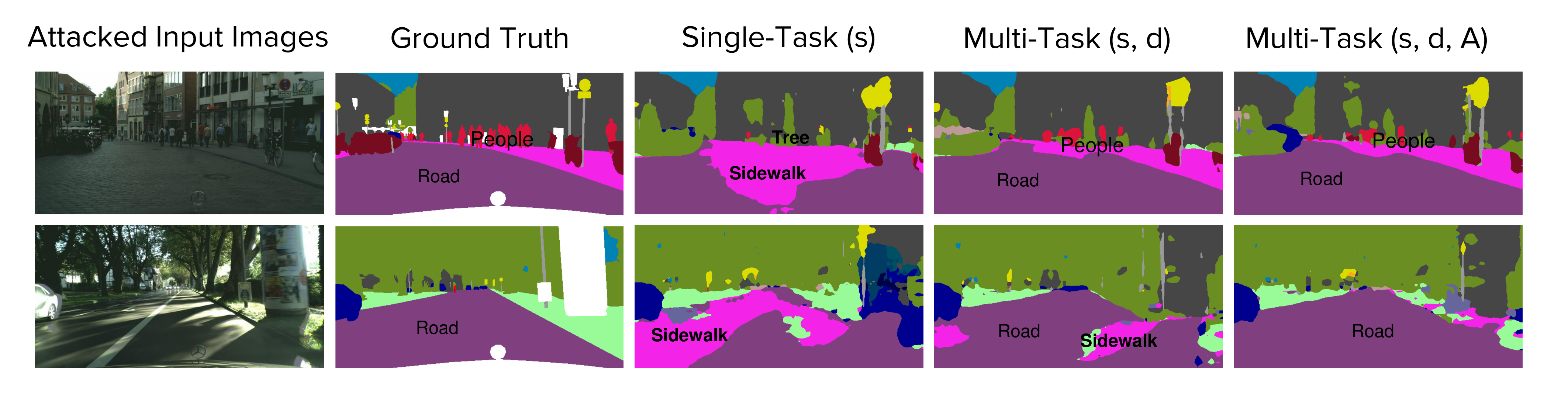}
	\vspace{-3em}
	\caption{We show model predictions on Cityscapes under \multitask attack. The single-task segmentation model misclassifies the `road' as `sidewalk' under attack, while the \multitask model can still segment it correctly. The \multitask models are more robust than the single-task trained model.}
	\label{fig:city-visualization-multi-attck}
\end{figure} 

\begin{figure*}[th]
	\centering
	\includegraphics[width=1\textwidth]{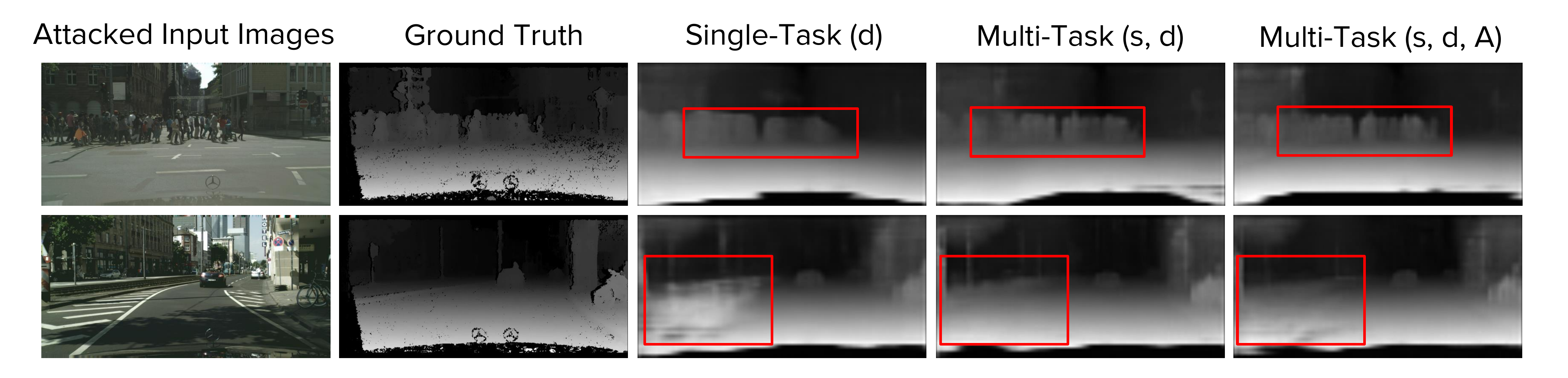}
	\caption{We show depth predictions of \multitask models under \multitask attacks. To emphasize the differences, we annotated the figure with red boxes where the errors are particularly noticeable. The \multitask trained model outperforms the single-task trained model under attack.}
	\label{fig:city-visualization-multi-attck-depth}
	\vspace{-5mm}
\end{figure*}

\noindent \textbf{\taskonomy.} The \taskonomy dataset \cite{taskonomy} consists of images of indoor scenes. We train on up to 11 tasks: semantic segmentation (s), depth euclidean estimation (D), depth zbuffer estimation (d), normal (n), edge texture (e), edge occlusion (E), keypoints 2D (k), keypoints 3D (K), principal curvature (p), reshading (r), and image reconstruction (A). We use the ``tiny'' version of their dataset splits \cite{Taskbank-split}. We resize the images to $256 \times 256$. 

% As previous work shows \cite{RobustSeg}, segmentation models are vulnerable even 

\subsection{Attack Methods}

% Given the above two types of adversarial objectives, we adopt the following attack strategies for norm bounded attacks. 

We evaluate the model robustness with $L_{\infty}$ bounded adversarial attacks, which is a standard evaluation metric for adversarial robustness \cite{madry}. We evaluate with four different attacks:

\textbf{FGSM:} We evaluate on the Fast Gradient Sign Method (FGSM) \cite{corr/GoodfellowSS14}, which generates adversarial examples $\x_{adv}$ by $\x_{adv}=\x+\epsilon \cdot \textrm{sign}(\nabla_\x \ell(F(\x),y))$. It is a single step, non-iterative attack.

\textbf{PGD:}   Following the attack setup for segmentation in \cite{RobustSeg},  we use the widely used attack PGD (Iteratively FGSM with random start \cite{madry}), set the number of iterations of attacks to $\min(\epsilon + 4, \lceil{1.25\epsilon}\rceil)$ and step-size $\alpha=1$. We choose the $L_{\infty}$ bound $\epsilon$ from $\{1, 2, 4, 8, 16\}$ where noise is almost imperceptible. Under $\epsilon=4$, we also evaluate the robustness using PGD attacks with $\{10, 20, 50, 100\}$ steps, which is a stronger attack compared to 5 steps attack used in \cite{RobustSeg}.

\textbf{MIM:} We also evaluate on MIM attack \cite{mim}, which adds momentum to iterative attacks to escape local minima and won the NeurIPS 2017 Adversarial Attack Competition.

%  We observe the models' performances converge after 100 steps of attack.

\textbf{Houdini:}
We evaluate the semantic segmentation task with the state-of-the-art Houdini attack \cite{houdini}, which directly attacks the evaluation metric, such as the non-differentiable mIoU criterion (mean Intersection over Union).

We do not use the DAG~\cite{DAG} attack for segmentation because it is an unrestricted attack without controlling $L_{\infty}$ bound. For all the iterative attacks, the step size is 1.

\begin{figure}[t]
	%   \centering
	%   \includegraphics[width=0.4\textwidth]{figures/cityscape_fig_norm.pdf}
	
	\centering
	\subfloat[Vulnerability (Dim)]{\label{subfig:AdvDim}\includegraphics[width=0.3\textwidth, trim={0cm 0 0cm 0},clip]{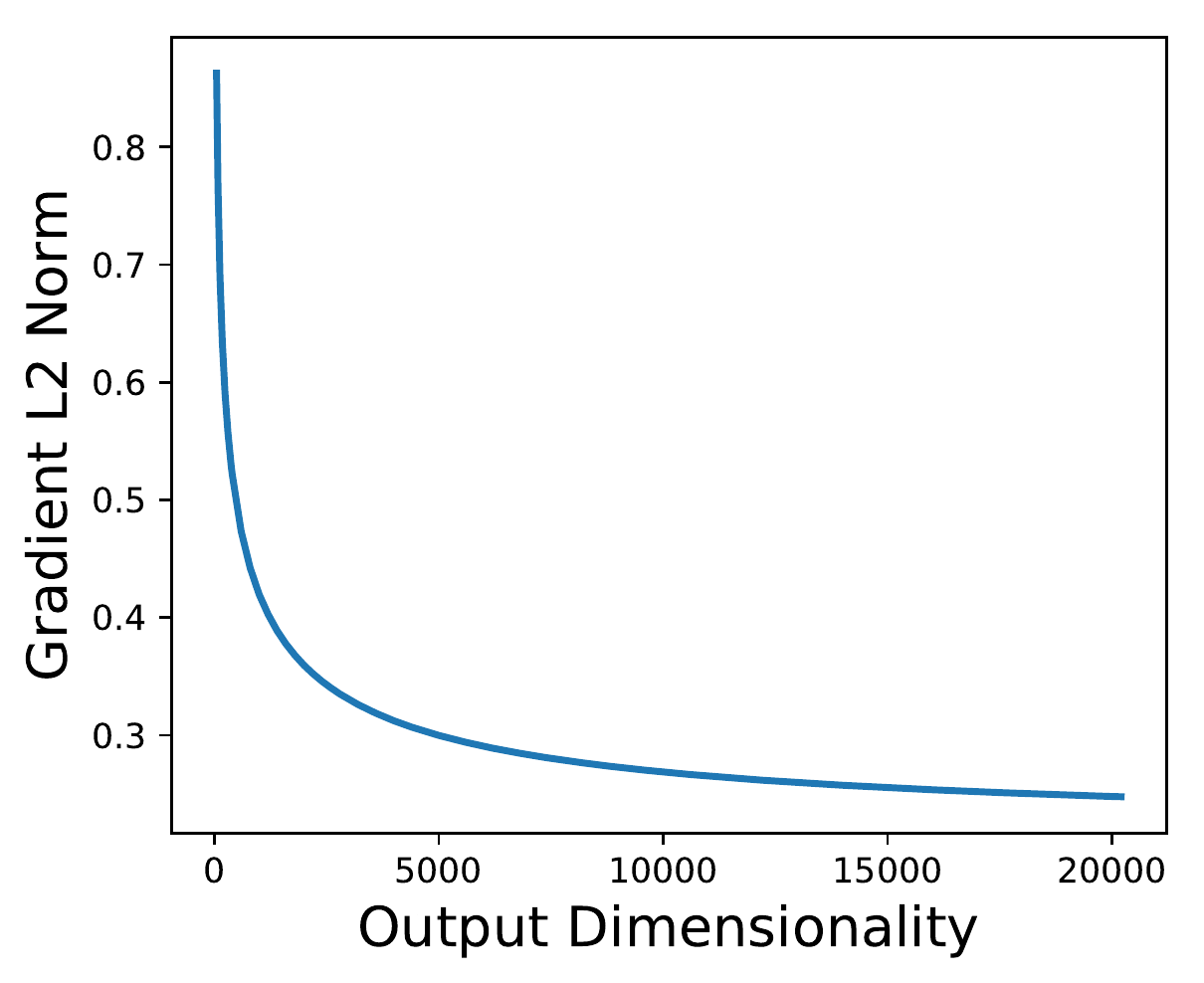}}
	\subfloat[Robust Accuracy]{\label{subfig:AccDim}\includegraphics[width=0.3\textwidth, trim={0cm 0 0cm 0},clip]{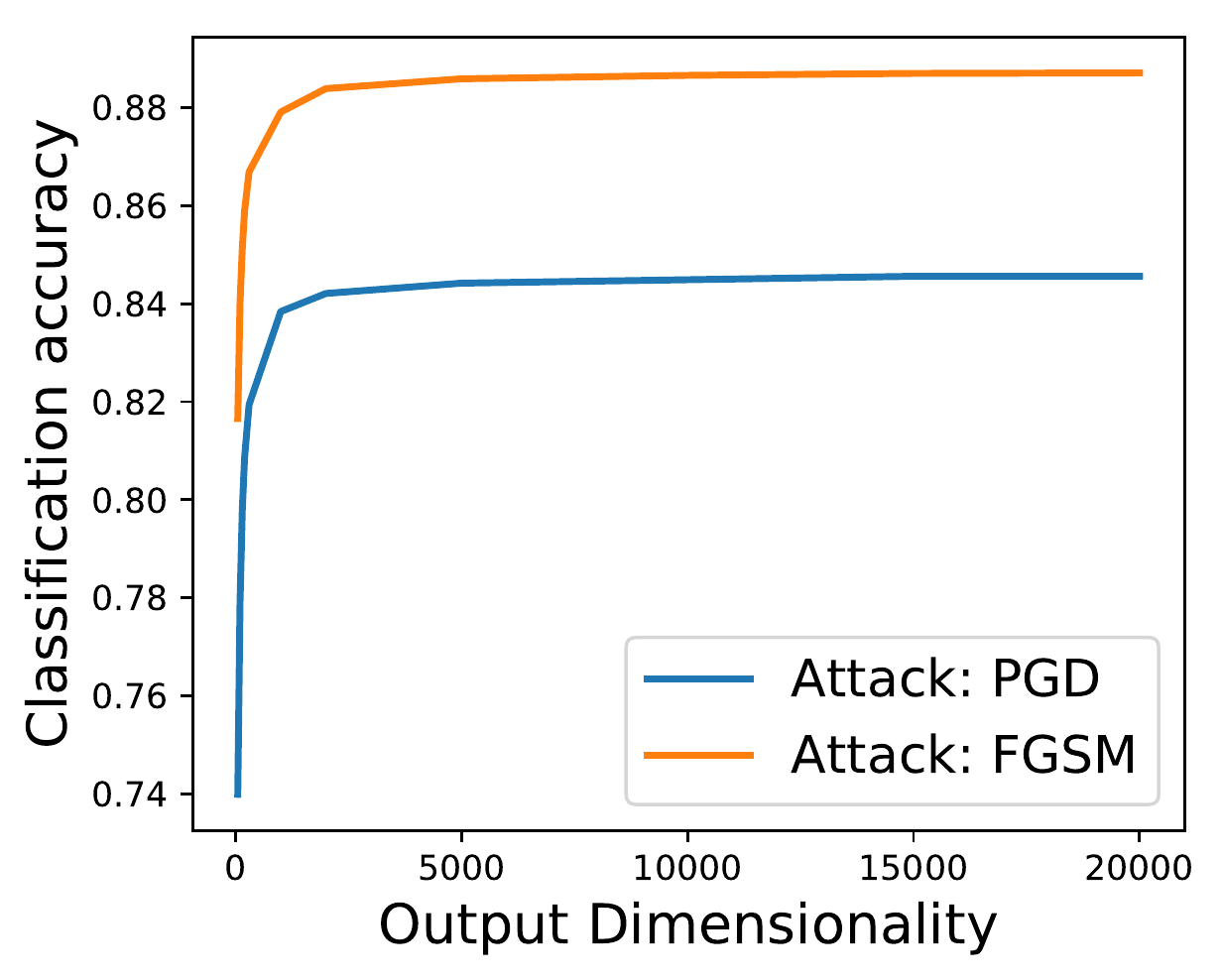}}
	\subfloat[Vulnerability (Task)]{\label{fig:taskonomy_grad_bar}\includegraphics[width=0.3\textwidth]{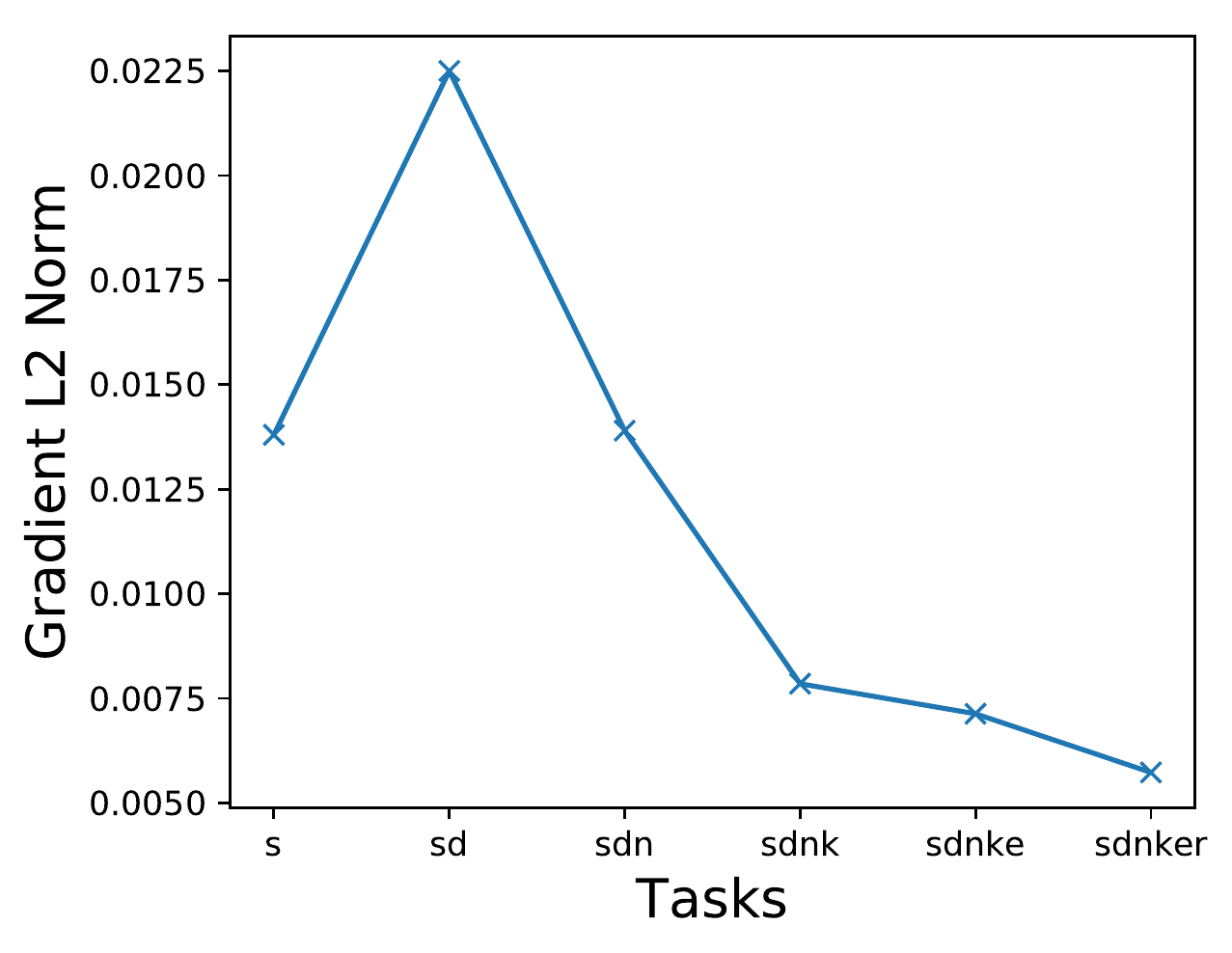}}
	
	\caption{\small{The effect of output dimensionality and number of tasks on adversarial robustness. We analyzed the pre-trained DRN model on \cityscape (a,b), and a \multitask model trained on \taskonomy (c).  The x-axis of (a,b) represents the output dimensionality, the x-axis of (c) shows the combination of multiple tasks. The y-axis of (a,c) is the L2 norm of the joint gradient and is proportional to the model's adversarial vulnerability. The y-axis of (b) is classification accuracy. The robust performances for (c) are shown in Fig \ref{fig:multiattack-taskonomy}. Increasing the output dimensionality or number of tasks improves the model's robustness.}}
	\label{fig:city-drn-output-dim}
	\vspace{-1mm}
\end{figure}

\subsection{\Multitask Models Against \Multitask Attack}
\label{sec:multi-attack}

\textbf{High Output Dimensionality as \Multitask.} Our experiment first studies the effect of a higher number of output dimensions on adversarial robustness. As an example, we use semantic segmentation. The experiment uses a pre-trained Dilated Residual Network (DRN-105) \cite{Yu2016, Yu2017} model on the \cityscape  dataset. To obtain the given output dimensionality, we randomly select a subset of pixels from the model output. We mitigate the randomness of the sampling by averaging the results over 20 random samples. Random sampling is a general dimension reduction method, which preserves the correlation and structure for high dimensional, structured data \cite{MC}.  Figure \ref{subfig:AdvDim} shows that the model's vulnerability (as measured by the norm of the gradients) decreases as the number of output dimension increases, which validates Theorem \ref{th:1}. 

% Without re-training the model for each output dimension, our method will get the same conclusion as other sub-sampling method, such as bilinear sub-sampling.

\begin{figure}[h!]
	\centering
	\vspace{-5mm}
	\subfloat[Segmentation mIoU $\uparrow$]{\label{aasxs}\includegraphics[width=0.35\textwidth, trim={0cm 0cm 0cm 0cm},clip]{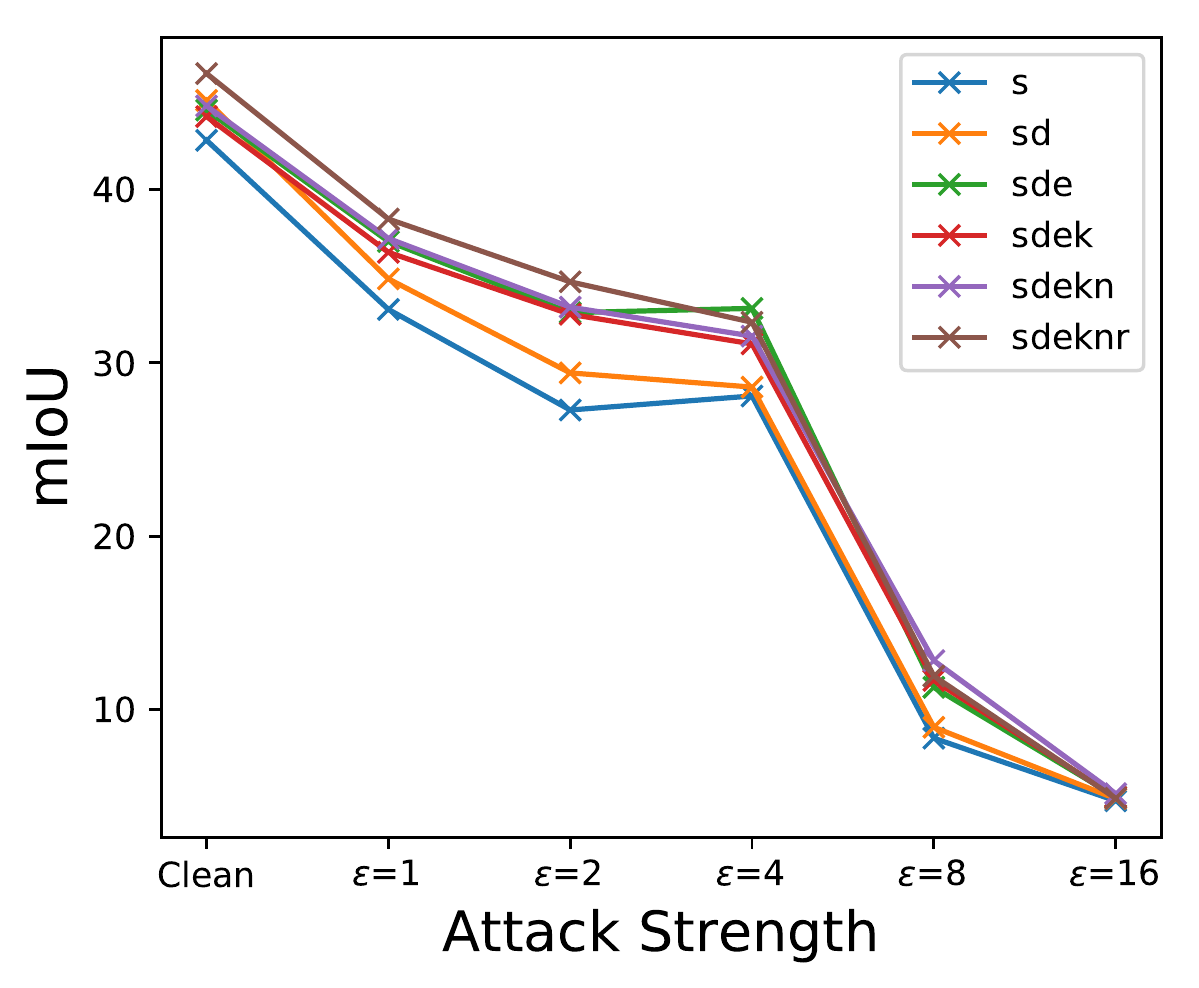}}
	\subfloat[Depth Abs Error $\downarrow$]{\label{aasxs}\includegraphics[width=0.35\textwidth, trim={0cm 0 0cm 0},clip]{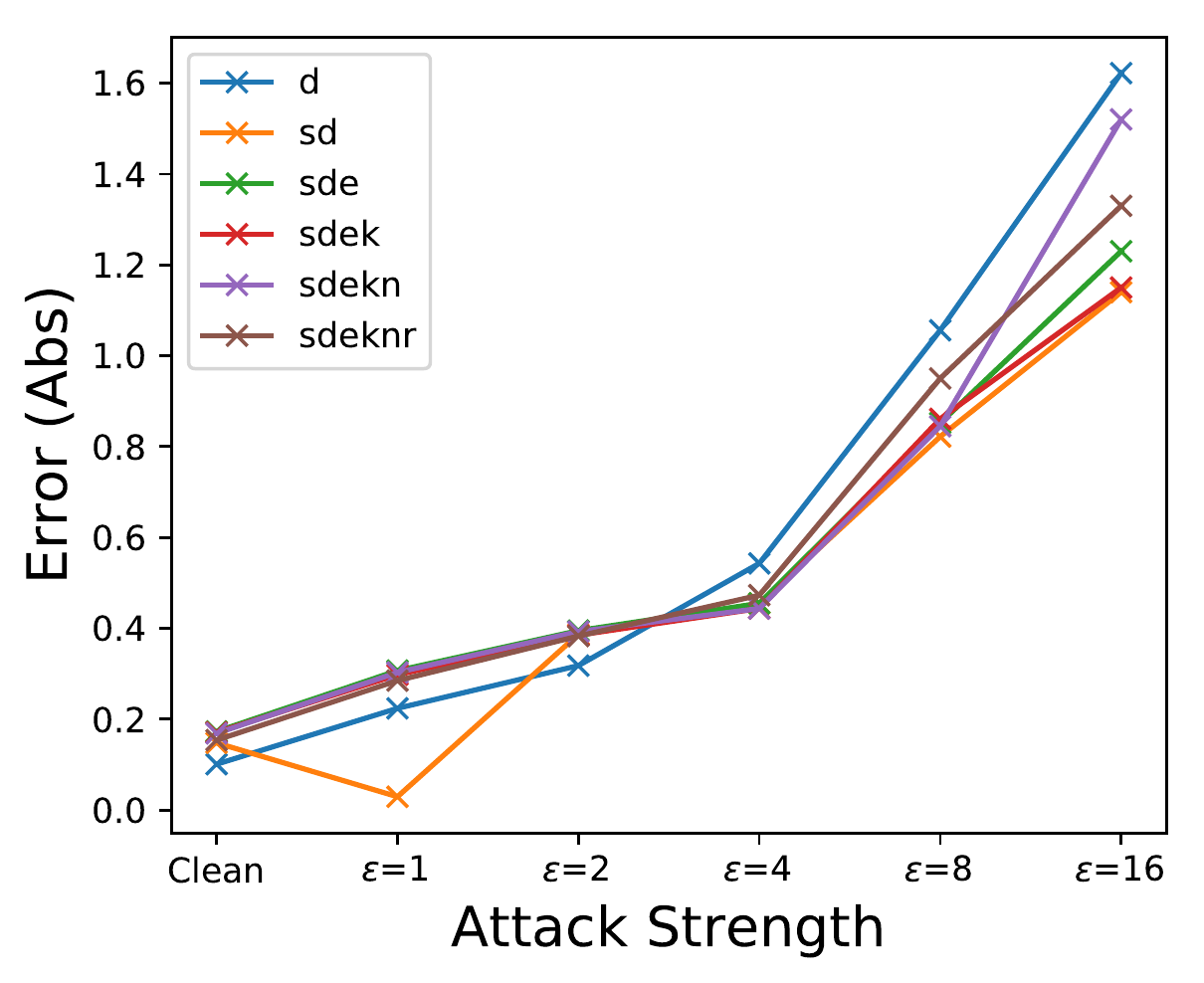}}
	\\
	\subfloat[Edge Detection MSE $\downarrow$]{\label{aasxs}\includegraphics[width=0.35\textwidth, trim={0cm 0 0cm 0},clip]{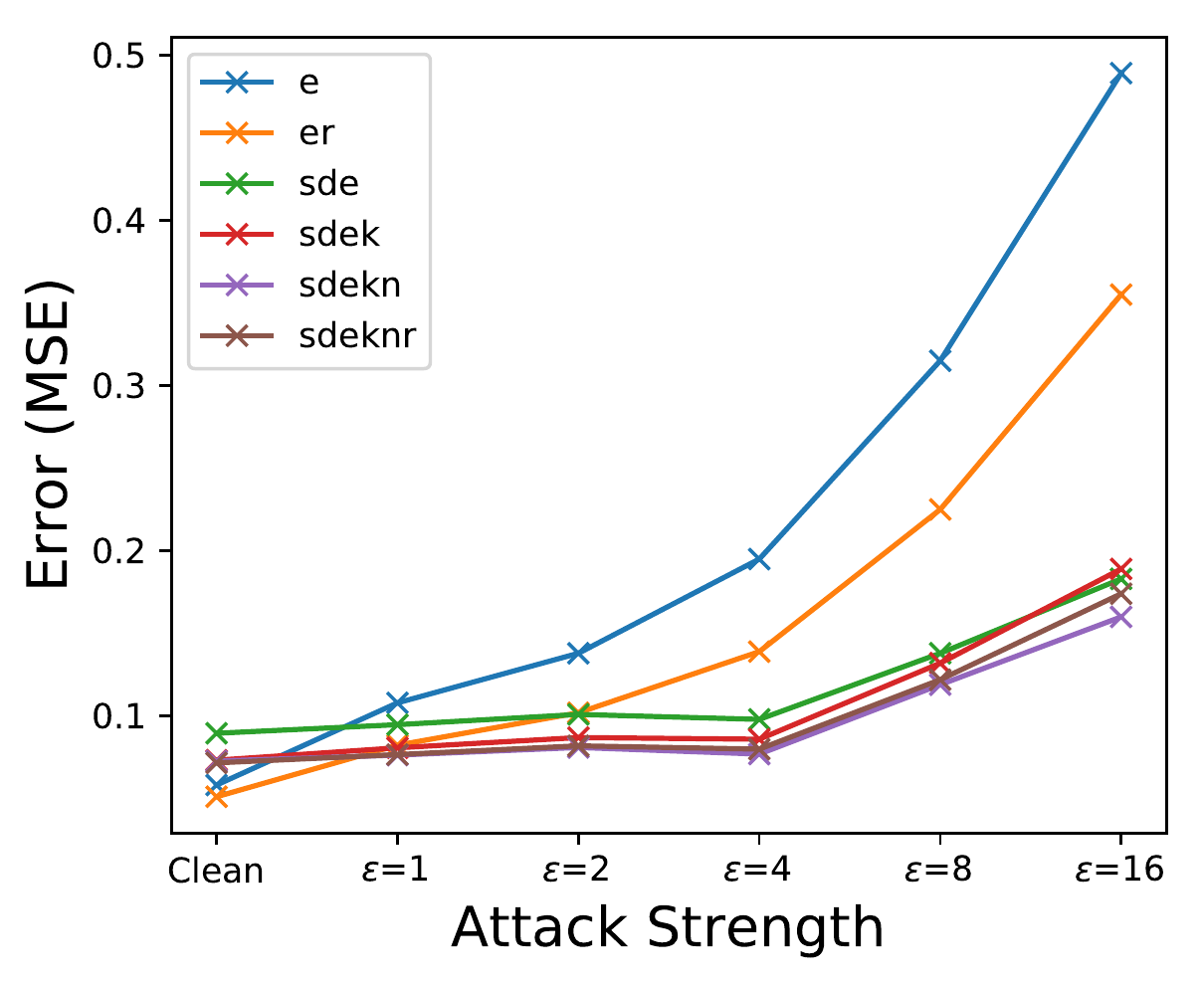}}
	\subfloat[Keypoints MSE $\downarrow$]{\label{aasxs}\includegraphics[width=0.35\textwidth, trim={0cm 0 0cm 0},clip]{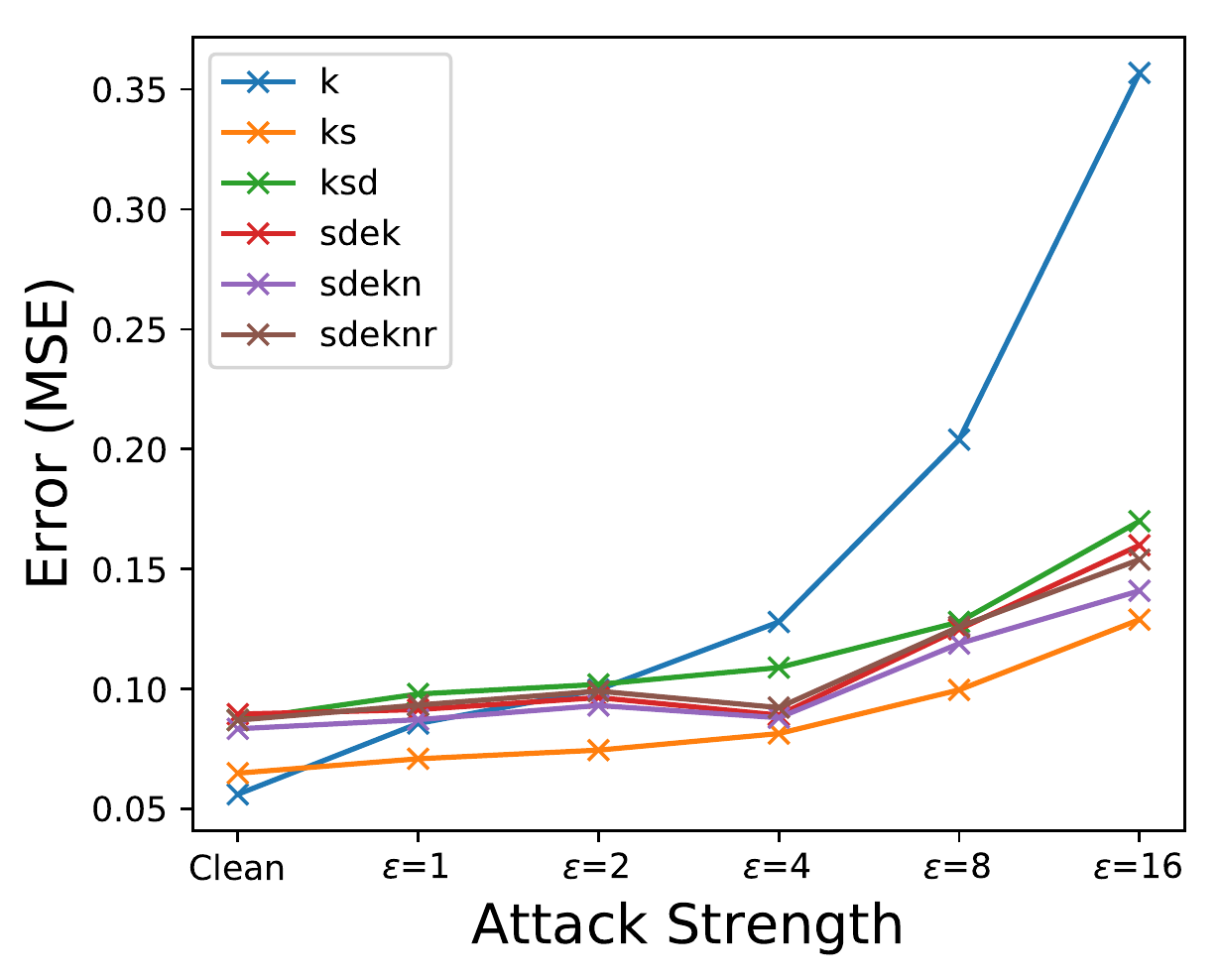}}
	
	\caption{\small{Adversarial robustness against \multitask attack on \taskonomy dataset. The x-axis is the attack strength, ranging from no attack (clean) to the strongest attack ($\epsilon=16$ PGD). For each subfigure, the y-axis shows the performance of one task under \multitask attack. $\uparrow$ means the higher, the better.  $\downarrow$ means the lower, the better. The \multitask model names are in the legend, we refer to the task by their initials, e.g., `sde' means the model is trained on segmentation, depth, and edge simultaneously. The \textcolor{blue}{blue}  line is the single-task baseline performance, the other lines are \multitask performance. The figures show that it is hard to attack all the tasks in a \multitask model simultaneously. Thus \multitask models are more robust against \multitask attacks.}}
	\label{fig:multiattack-taskonomy}
	\vspace{-4mm}
\end{figure}

% Each \emph{column} in the figure, ranging from no attack (clean) on the left to the strongest attack ($\epsilon=16$ PGD) on the right, shows the model's performance under the same strength of attack. Each \emph{row} in the figure shows the performance of the same task. Each bar in the sub-figure corresponds to one multi-task model specified by the rotated text, e.g., `sde' means the model is trained on segmentation, depth, and edge simultaneously. The leftmost bar for each subfigure is single task baseline. The \textcolor{green}{green} bar indicates multi-task model are better than single task model. Our main finding is that, as we increase the strength of the attack, multi-task trained models (Column 4-6) are more robust than single-task trained models (Column 1). For weak attack, there are outliner in Column 2-3 ($\epsilon=1,2$), where the model robust performance correlates more with the clean performance and thus does not show improvement. Thus, training on multiple tasks improves the overall robustness of the model against strong multi-task attack.

Besides the norm of gradient, we measure the performance under FGSM \cite{corr/GoodfellowSS14} and PGD \cite{madry} adversarial attacks, and show that it improves as output dimensionality increases (Figure \ref{subfig:AccDim}). Notice when few pixels are selected, the robustness gains are faster. This is because with fewer pixels: (1) the marginal gain of the inverse function is larger; and (2) the select pixels are sparse and tend to be far away and uncorrelated to each other. The correlation between the output pixels compounds as more nearby pixels are selected, which slows down the improvements to robustness. The results demonstrate that models with higher output dimension/diversity are inherently more robust against adversarial attacks, consistent with the observation in \cite{RobustSeg, houdini} and our Theorem \ref{th:1}.

%  \begin{figure} 
%   \centering
% %   \includegraphics[width=0.4\textwidth]{figures/cityscape_fig_norm.pdf}
%   \includegraphics[width=0.5\textwidth]{figures/multi-vulnerability-grad-totalloss.pdf}
%   \caption{Adversarial vulnerability of learned \multitask model on Taskonomy. The y axis shows the L2 norm of the joint gradient; the x axis shows models trained with different number of tasks, including \textbf{s}emantic segmentation, \textbf{d}epth, \textbf{n}ormal, \textbf{k}eypoints estimation (2D), \textbf{e}dge, and \textbf{r}eshading. As the number of output tasks increases, the adversarial vulnerability decreases, with the exception of depth, possibly due to its large range of output values.} \label{fig:total_loss_multi-task-grad}
% \end{figure}

\noindent \textbf{Number of Tasks.} We now consider the case where the number of tasks increases, which is a second factor that increases output dimensionality. We evaluate the robustness of \multitask models on the \cityscape and \taskonomy datasets. We equally train all the tasks with the shared backbone architecture mentioned in Section \ref{sec:adv_setting}. On \cityscape, we use DRN-105 model as the architecture for encoder and decoder; on \taskonomy, we use Resnet-18 \cite{taskonomy}. Each task has its own decoder. For the \cityscape dataset, we start with training only the semantic segmentation task, then add the depth estimation and input reconstruction task. For the \taskonomy dataset, following the setup in \cite{multi-task-onomy}, we start with only semantic segmentation, and add depth estimation, normal, keypoints 2D, edge texture, and reshading tasks to the model one by one.  In our figures and tables, we refer to these tasks by the task's first letter. 

Figure \ref{fig:taskonomy_grad_bar} shows the L2 norm of the joint gradient for many tasks, which measures the adversarial vulnerability. Overall, as we add more tasks, the norm of the joint gradient decreases, indicating improvement to robustness \cite{first-order}. The only exception is the depth estimation task, which we believe is due to the large range of values (0 to $+\infty$) that its outputs take. Empirically, a larger output range leads to a larger loss, which implies a larger gradient value.

% This is non-travial, since the trend is not clear unless strong enough attack is evaluated. 

% shows that the performance under attack for \taskonomy segmentation trained on 6 tasks improves by up to .
% Overall, as the number of tasks increases, the performance of each individual task under attack improves, indicating that they become more robust.

\begin{figure}[t]
	\centering
	\includegraphics[width=1\textwidth]{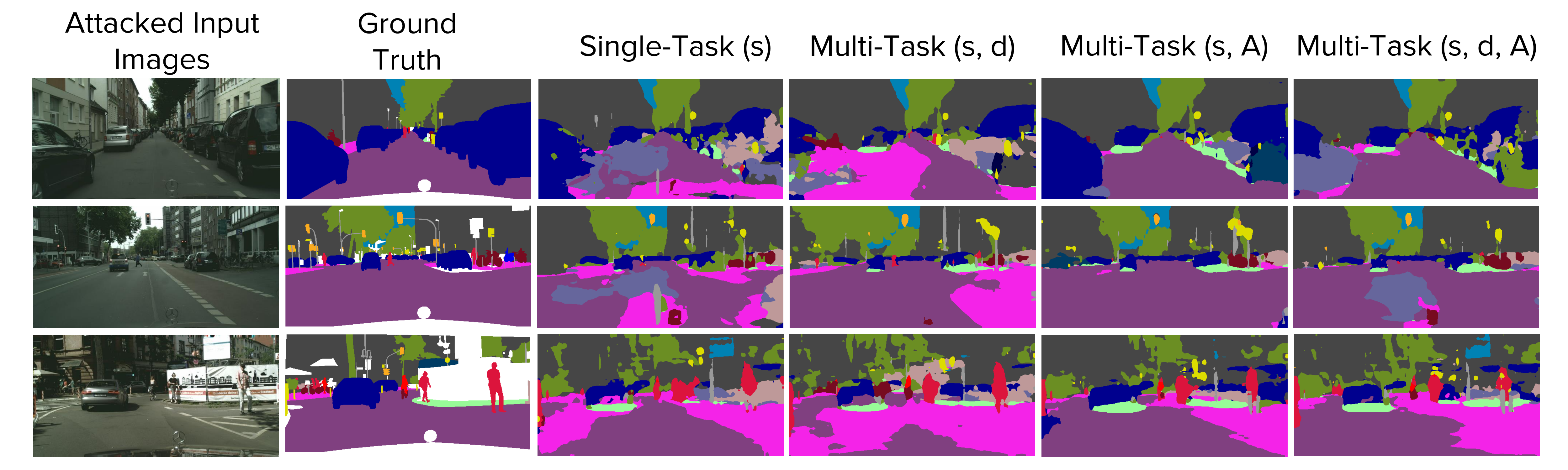}
	\caption{ \small{Performance of single-task attack for \multitask models trained on Cityscapes. We show segmentation under attack for single-task and three \multitask models. The \multitask trained model out-performs the single-task trained model.}}
	\label{fig:city-visualization-single-attack}
	\vspace{-5mm}
\end{figure}

\begin{table}[b]
	\begin{center}
		\label{exp:city-sig}
		\scriptsize
		\centering
		\begin{tabular}{l|c|cc}
			\toprule
			& Baseline  & \multicolumn{2}{c}{\Multitask}\\
			Training Tasks & s   & sd  & sdA \\
			\midrule  
			Clean SemSeg $\uparrow$ & 44.77 & \textbf{46.53} & 45.82 \\
			PGD  SemSeg $\uparrow$ & 15.75 & 16.01  & \textbf{16.36} \\
			\bottomrule
		\end{tabular}
		\quad
		\begin{tabular}{l|c|cc}
			\toprule
			& Baseline & \multicolumn{2}{c}{\Multitask}\\
			Training Tasks   & d  & sd  & sdA \\
			\midrule  
			Clean Depth $\downarrow$  & 1.82 &  \textbf{1.780} & \underline{1.96} \\
			PGD Depth $\downarrow$ & 6.81  & 6.08 & \textbf{5.81}\\
			
			\bottomrule
		\end{tabular}
	\end{center}
	\caption{\small{The models' performances under \multitask PGD attack and clean images on \cityscape using DRN-D-105 \cite{Yu2017} . The \textbf{bold} demonstrate the better performance for each row, \underline{underline} shows inferior results of \multitask learning. The results show that \multitask models are overall more robust under \multitask attack.
	}} \label{tab:city-multi-attack-all}
	
\end{table}
% Table \ref{exp:taskonomy-multi} shows that the performance under attack for \taskonomy segmentation trained on 6 tasks improves mIoU score more than 2.6$\times$ over the baseline. 

We additionally measure the robust performance on different \multitask models under \multitask attacks. Following the setup in \cite{RobustSeg}, we enumerate the $\epsilon$ of the $L_{\infty}$ attack from 1 to 16. Figure \ref{fig:multiattack-taskonomy} shows the robustness of \multitask models using \taskonomy, where the adversarial robustness of \multitask models are better than single-task models, even if the clean performance of \multitask models may be lower. 
We also observe some tasks gain more robustness compared to other tasks when they are attacked together, which suggests some tasks are inherently harder to attack. Overall, the attacker cannot simultaneously attack all the tasks successfully, which results in improved overall robustness of \multitask models. In Table \ref{tab:city-multi-attack-all}, we observe the same improvement on \cityscape. Qualitative results are shown in Figure \ref{fig:city-visualization-multi-attck} and Figure \ref{fig:city-visualization-multi-attck-depth}. Please see the supplemental material for additional results. 

% Consistent with prior work \cite{multi-task-onomy}, the clean performance of multi-task models can be inferior to single-task models. For example, the depth and edge task performs better under single task training (Column 1).

% This observation unveils more desirable property of \multitask learning. 

\begin{table}[ht]
	\begin{center}
		\label{exp:city-sig}
		\scriptsize
		% \vspace{}
		% \small
		\centering
		\begin{tabular}{cl|c|ccc|c|ccc}
			\toprule
			& & \multicolumn{4}{c|}{SemSeg mIoU Score  $\uparrow$} & \multicolumn{4}{c}{Depth Abs Error $\downarrow$} \\
			\midrule  
			& & Baseline & \multicolumn{3}{c|}{\Multitask Learning}  & Baseline & \multicolumn{3}{c}{\Multitask Learning}\\
			& Training Tasks $\xrightarrow{}$ & s & sd & sA & sdA & d  & ds  & dA  & dAs \\
			& $\lambda_a$ &  &  0.001 & 0.001 & 0.001 & & 0.1 & 0.1 & 0.01 \\
			\midrule 
			&Clean & 48.58 & 48.61 & \textbf{49.61} & \underline{48.19} &1.799 & \textbf{1.792} &  \underline{1.823} & 1.798 \\
			\midrule
			\parbox[t]{3mm}{\multirow{7}{*}{\rotatebox[origin=c]{90}{Attacks}}}
			&FGSM & 26.35 &	\underline{26.28} &	\textbf{26.79} &	26.71  &3.16 &	3.01&	\textbf{3.00} &	3.24 \\
			
			%  BIM \cite{BIM} & 16.29 & \textbf{16.41} & \underline{\textbf{18.05}}& \textbf{17.47} &&&& & 5.475 & \textbf{4.9636} & \underline{\textbf{4.9102}} & \textbf{5.2904} \\
			
			&PGD10  & 13.04&	13.64&	\textbf{14.76}	&14.48  &6.96&	6.15&	\textbf{6.03}&	6.59 \\
			&PGD20& 11.41&	11.98&\textbf{12.79}&	12.73 &8.81&	\textbf{7.70}	&7.64&	8.38 \\
			&PGD50  &  10.49&	10.95&	11.68 & \textbf{11.86} &10.23&	\textbf{9.07}&	9.12&	9.81 \\
			&PGD100  &  10.15&	10.51&	11.22& \textbf{11.52}  &10.8&	\textbf{9.69}&	9.74&	10.41 \\
			&MIM100 & 9.90 &	10.17&	10.93& \textbf{11.24} & 12.04	&\textbf{10.72}&	10.97&	11.69\\ 
			&Houdini100 & 5.04 &	5.14 &	\textbf{6.24} & 6.21 & - &- & -& -\\
			\bottomrule
		\end{tabular}
	\end{center}
	\caption{\small{Model's robust performance under $L_{\infty}=4$ bounded single-task attacks on \cityscape. Each column is a DRN-22 model trained on a different combination of tasks, where ``s,''``d,''and``A''denote segmentation, depth, and auto-encoder, respectively.  $\uparrow$ means the higher, the better. $\downarrow$ means the lower, the better. \textbf{Bold} in each row, shows the best performance under the same attack.  \Multitask learning models out-perform single-task models except for the \underline{underlined} ones. While nearly maintaining the performance on clean examples, \multitask models are consistently more robust under strong adversarial attacks. }
	} \label{tab:city-single-attack} 
\end{table}

\subsection{\Multitask Models Against Single-Task Attacks}
\label{sec:single-attack}

Following the setup for \multitask learning in \cite{luong2015multitaskshare, Liu_2019_MTLshare, lee2019MTLgeneralization}, we train the \multitask models using a main task and auxiliary tasks, where we use $\lambda=1$ for the main task and $\lambda_a$ for the auxiliary tasks. We then evaluate the robustness of the main task under single-task attacks. On \cityscape, the main and the auxiliary tasks share 16 layers of an encoding backbone network. The decoding network for each individual task has 6 layers. For all the models, we train for 200 epochs. For adversarial robustness evaluation, we use strong attacks including PGD100 and MIM100 for attacking the segmentation accuracy\footnote{Suffixed number indicates number of steps for attack}, and use 100 steps Houdini \cite{houdini} to attack the non-differentiable mIoU of the Segmentation model directly. We do not use Houdini to attack the depth because the L1 loss for depth is differentiable and does not need any surrogate loss.
The results in Table \ref{tab:city-single-attack} show that \multitask learning improves the segmentation mIoU by 1.2 points and the performance of depth estimation by 11\% under attack, while maintaining the performance on most of the clean examples. Qualitative results are in Figure \ref{fig:city-visualization-single-attack}.

\begin{figure}[h!]
	
	\centering
	\subfloat[Performance Under Attack]{\label{aasxs}\includegraphics[width=1.0\textwidth]{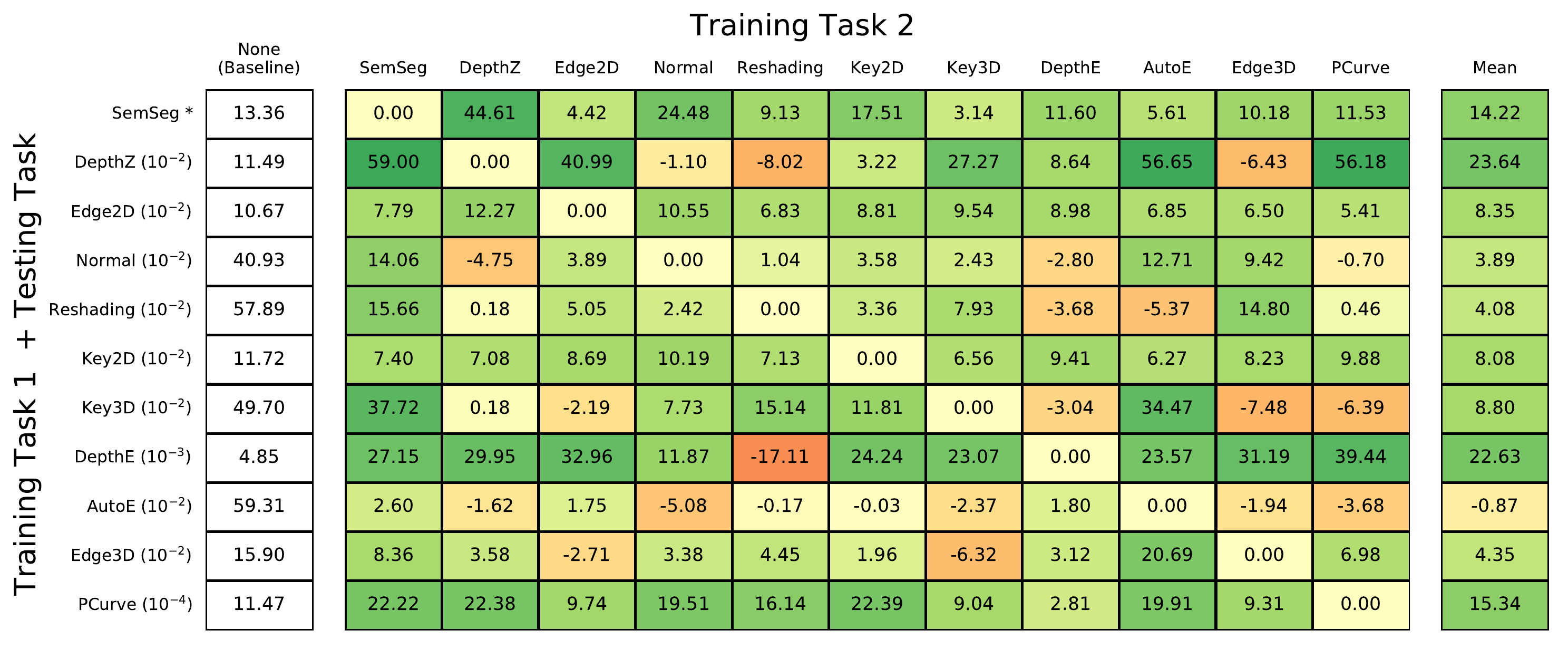}}
	%   \vspace{-5mm}
	\\
	\subfloat[Performance on Clean Examples]{\label{aasxs}\includegraphics[width=1.0\textwidth]{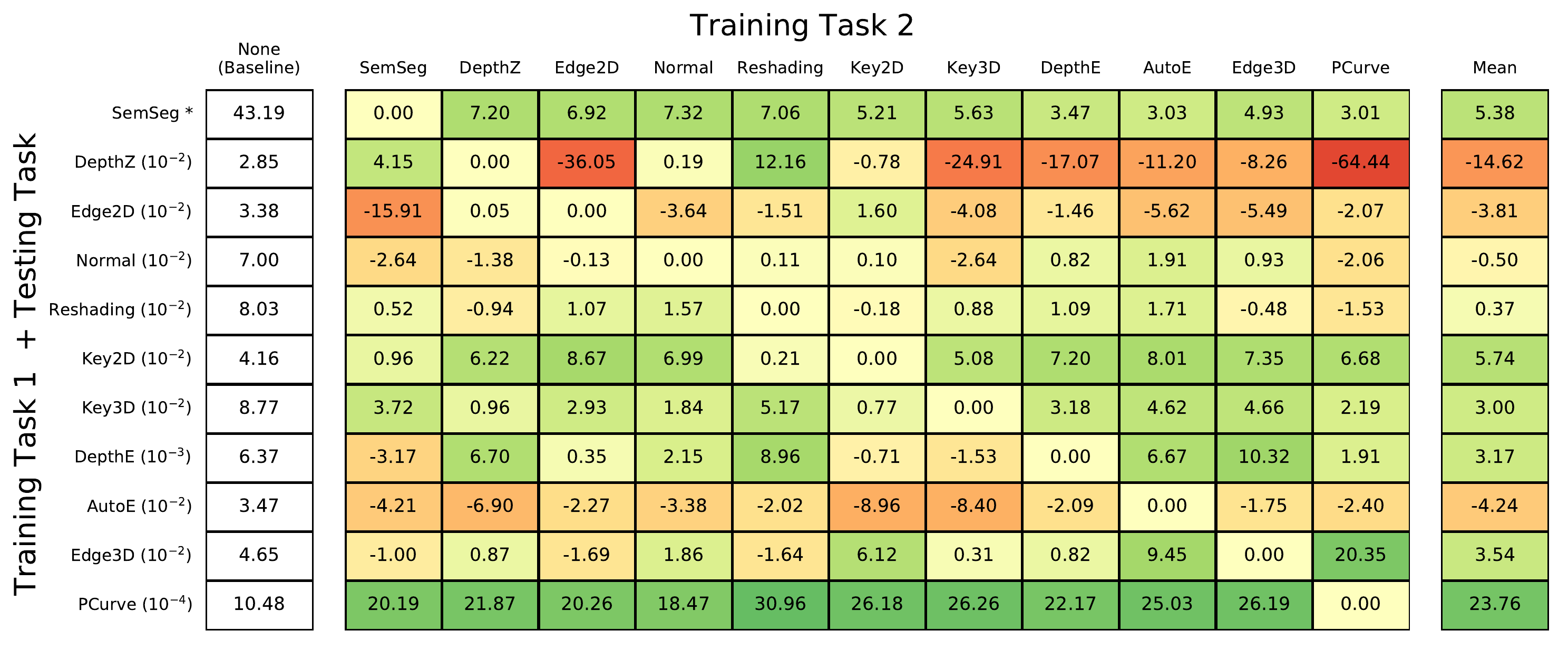}}
	\caption{\small{We consider models trained on two tasks. In each matrix, the rows show the first training task and the testing task. The columns show the auxiliary training task. The first column without color shows the absolute value for the baseline model (single-task).  The middle colored columns show the relative improvement of \multitask models over the single-task model in percentage. The last colored column shows the average relative improvement. We show results for both (a) adversarial and (b) clean performance. \Multitask learning improves the performance on clean examples for 70/110 cases, and the performance on adversarial examples for 90/110 cases. While \multitask training does not always improve clean performance, we show \multitask learning provides more gains for adversarial performance.}}
	\label{fig:attack-matrix-taskonomy}
	\vspace{-5mm}
\end{figure}

On the \taskonomy dataset, we conduct experiments on 11 tasks. Following the setup in \cite{multi-task-onomy}, we use ResNet-18 \cite{ResNet} as the shared encoding network, where each individual task has its own prediction network using the encoded representation. We train single-task models for each of the 11 tasks as baselines. We train a total of 110 \multitask models --- each main task combined with 10 different auxilliary tasks --- for 11 main tasks. We evaluate both the clean performance and adversarial performance. $\lambda_a$ is either 0.1 or 0.01 based on the tasks. We use PGD attacks bounded with $L_{\infty} = 4$ with 50 steps, where the step size is 1. The attack performance plateaus for more steps. Figure \ref{fig:attack-matrix-taskonomy} shows the performance of the main task on both clean and adversarial examples. While maintaining the performance on clean examples (average improvement of 4.7\%), \multitask learning improves 90/110 the models' performance under attacks, by an average of 10.23\% relative improvement. Our results show that one major advantage of \multitask learning, which to our knowledge is previously unknown, is that it improves the model's robustness under adversarial attacks.

\subsection{\Multitask Learning Complements Adversarial Training}
\label{sec:adv-train}

\begin{table}
\begin{center}
\label{exp:city-sig}
\scriptsize
% \small
    \centering
    \begin{tabular}{cl|c|ccc|c|ccc}
         \toprule
         & & \multicolumn{4}{c|}{SemSeg mIoU Score  $\uparrow$}  & \multicolumn{4}{c}{Depth Abs Error $\downarrow$} \\
         \midrule  
         & & Baseline & \multicolumn{3}{c|}{\Multitask Learning}   & Baseline & \multicolumn{3}{c}{\Multitask Learning}\\
         & Training Tasks $\xrightarrow{}$ & s & sd & sA & sdA  & d  & ds  & dA  & dAs \\
         \midrule  
        %  & Clean (no adv train) & 48.58 & 48.61 & \textbf{49.61} & \underline{48.19} &1.80 & \textbf{1.79} &  \underline{1.82} & 1.80 \\
         & Clean & 41.95 & 43.27 & \textbf{43.65} & 43.26 & 2.24 & \textbf{2.07} & 2.15 & 2.15\\
         \midrule
         \parbox[t]{3mm}{\multirow{4}{*}{\rotatebox[origin=c]{90}{Attacks}}}  

         & PGD50  & 19.73 &	\textbf{22.08} &	20.45 &	21.93 & 2.85&	\textbf{2.61} & 2.75 & 2.67 \\
         & PGD100  & 19.63 &	\textbf{21.96} &	20.31 &	21.83  & 2.85 &	\textbf{2.61} & 2.75 & 2.67  \\
         & MIM100 & 19.54 & \textbf{21.89} & 20.20 & 21.74 & 2.85 & \textbf{2.61} & 2.75 & 2.67\\ 
         & Houdini100 & 17.05 & \textbf{19.45} & 17.36 & 19.16 & - & - & - & -\\
         \bottomrule
    \end{tabular}
\end{center}
\caption{\small{Adversarial robustness of adversarial training models under $L_{\infty}=4$ bounded attacks on \cityscape. Each column is a model trained on a different combination of tasks. ``s,''``d,''and``A''denote segmentation, depth, and auto-encoder respectively.  $\uparrow$ indicates  the higher, the better. The $\downarrow$ indicates the lower, the better. \textbf{Bold} shows the best performance of the same task for each row. \Multitask learning improves both the clean performance and robustness upon single-task learning.  }
} \label{tab:city-adv-multi}
\vspace{-4mm}
\end{table}

% Clean (no adv train) shows the performance of models without adversarial training on clean examples, while the following rows are performance of adversarially trained models.
%   Despite that all models drop clean performance after adversarial training, m

% In the above section, we show that multi-task/versatile models improves adversarial robustness while maintaining the clean accuracy. Adversarial training is a widely used strategy for improving adversarial robustness of a model through adversarial training \cite{madry}. One major disadvantage of adversarial training is that is decrease the model's performance under the clean examples. This prohibits the deployment of adversarial training algorithm to realistic applications, where the performance under clean examples is the major concern. 

We study whether \multitask learning helps adversarial robust training. We use DRN-22 on the \cityscape dataset, and train both single-task and \multitask models for 200 epoch under the same setup. The single-task model follows the standard adversarial training algorithm, where we train the model on the generated single-task (segmentation) adversarial attacks. For the \multitask adversarial training, we train it on the generated \multitask attack images for both semantic segmentation and the auxiliary task. Details are in the supplementary material. 
Table \ref{tab:city-adv-multi} shows that \multitask learning improves the robust performance of both clean examples and adversarial examples, where segmentation mIoU improves by 2.40 points and depth improves by 8.4\%.

\section{Conclusion}

% This paper shows that learning versatile models improves robustness. 
The widening deployment of machine learning in real-world applications calls for versatile models that solve multiple tasks or produce high-dimensional outputs. Our theoretical analysis explains that versatile models are inherently more robust than models with fewer output dimensions.  Our experiments on real-world datasets and common computer vision tasks measure improvements in adversarial robustness under attacks.  Our work is the first to connect this vulnerability with \multitask learning and hint towards a new direction of research to understand and mitigate this fragility.

{
\textbf{Acknowledgements:} This work was in part supported by Amazon Research Award; NSF grant CNS-15-64055; NSF-CCF 1845893; NSF-IIS 1850069; ONR grants N00014-16-1- 2263 and N00014-17-1-2788; a JP Morgan Faculty Research Award; a DiDi Faculty Research Award; a Google Cloud grant; an Amazon Web Services grant. The authors thank Vaggelis Atlidakis, Augustine Cha, D\'idac Sur\'is, Lovish Chum, Justin Wong, and Shunhua Jiang for valuable comments.
}

\clearpage
% ---- Bibliography ----
%
% BibTeX users should specify bibliography style 'splncs04'.
% References will then be sorted and formatted in the correct style.
%
\bibliographystyle{splncs04}
\bibliography{eccv2020submission}

\externaldocument[S-]{supplementary}

\newpage
\begin{subappendices}

\section{Proof for Theoretical Analysis}

 ~~~~~ We present theoretical analysis to quantify how much multi-task learning improves model's overall adversarial robustness.

% both multi-task learning model which produces multiple outputs at the same time (see Sec 5.1), or ensemble models consist of multiple different models (See Sec 5.4 in the paper). The performance of ensemble models improves the robustness, because each individual neural network in the ensemble in independently trained, which makes the model less correlated to each other. As a result, the gradient produced by each individual network are less correlated, which can also be explained with our theory. Moreover, the ensemble of multiple models trained on different auxiliary tasks improves the diversity (decrease the correlation) of different tasks, which further improves the robustness (see Figure 6 in paper) according to theorem \ref{th:1}. 

\begin{definition}
Given  classifier $F$, input $\x$, output target $\y$, and loss $\mathcal{L}(\x, \y) = \ell(F(\x), \y)$, the feasible adversarial examples lie in a $p$-norm bounded ball with radius $r$, $B(\x, r) := \{\x_{adv}, ||\x_{adv}-\x||_p<r\}$. Then adversarial vulnerability of a classifier over the whole dataset is 
\[\mathbb{E}_{\x} [\Delta \mathcal{L}(\x, \y, r)] = \mathbb{E}_{\x} [\max_{||\delta||_p < r}{|\mathcal{L}(\x,\y) - \mathcal{L}(\x+\delta,\y)|}].\]
% we call a successful $\ell_p$-norm bounded attack by $\forall c \in \{1...M\} \ \  F_c(x+\delta) \neq y_c$, s.t. $||\delta||_p < \epsilon$
\end{definition}

\noindent $\Delta \mathcal{L}$ captures the change of output loss given a change in input. Intuitively, a robust model should have a smaller change in loss given a perturbation of the input. Given the adversarial noise is imperceptible, i.e., $r \rightarrow 0$, we can approximate $\Delta \mathcal{L}$ with a first-order Taylor expansion, where

\[|\mathcal{L}(\x,\y) - \mathcal{L}(\x+\delta,\y)| = |\partial_{\x} \mathcal{L}(\x, \y) \delta + O(\delta)|
\]

\begin{lemma}
For a given neural network $F$ that predicts multiple tasks, the adversarial vulnerability is
\[\mathbb{E}_{\x} [\Delta \mathcal{L}(\x, \y, r)] \approx  \mathbb{E}_{\x}\left[||\partial_{\x} \mathcal{L}_{all}(\x, \overline{\y})||_q \right] \cdot ||\delta||_p \propto \mathbb{E}_{\x}\left[||\partial_{\x} \mathcal{L}_{all}(\x, \overline{\y})||_q \right] \]
\end{lemma}

\begin{proof}
According to the definition of dual norm:

\[ \Delta \mathcal{L} \approx \max_{||\delta||_p < r} |\partial_{\x} \mathcal{L}(\x, \y) \delta | = ||\partial_{\x} \mathcal{L}_{all}(\x, \overline{\y})||_q \cdot || \delta ||_p
\]

\[
\mathbb{E}_{\x}[\Delta \mathcal{L}]  \approx \mathbb{E}_{\x}\left[||\partial_{\x} \mathcal{L}_{all}(\x, \overline{\y})||_q \right] \cdot ||\delta||_p 
\]

\noindent where $q$ is the dual norm of $p$, which satisfies $\frac{1}{p} + \frac{1}{q} = 1$ and $1 \leq p \leq \infty$.

Once given the $p$-norm bounded ball, i.e., $||\delta||_p $ is constant, we get

\[
\mathbb{E}_{\x}[\Delta \mathcal{L}]  \propto \mathbb{E}_{\x}\left[||\partial_{\x} \mathcal{L}_{all}(\x, \overline{\y})||_q \right] 
\]

\end{proof}{}

\begin{theorem}\label{th:1}
\textbf{(Adversarial Vulnerability of Model for Multiple Correlated Tasks)}
If the selected output tasks are correlated with each other such that the covariance between the gradient of task $i$ and task $j$ is $\mathrm{Cov}(\bf{r}_i, \bf{r}_j)$, and the gradient for each task is i.i.d. with zero mean (because the model is converged), then adversarial vulnerability of the given model is proportional to 
\[\frac{\sqrt{(1 + \frac{2}{M} \sum_{i=1}^{M} \sum_{j=1}^{i-1} \frac{\mathrm{Cov}(\bf{r}_i, \bf{r}_j)}{\mathrm{Cov}(\bf{r}_i, \bf{r}_i)})}}{\sqrt{M}} \]
where $M$ is the number of output tasks selected.
\end{theorem}

\begin{proof}
Denote the gradient for task $c$ as $\bf{r}_c$, i.e., 
\[
\bf{r}_c = \partial_{\x} \mathcal{L}_c(\x, \y_c)
\] 

 We define the joint gradient vector $\R$ as follows:
\[
    \R = \partial_{\x} \mathcal{L}_{all}(\x, \overline{\y}) = \partial_{\x} \left(\frac{1}{M} \sum_{c=1}^{M}{\mathcal{L}_c(\x, \y_c)}\right) =  \frac{1}{M} \sum_{c=1}^{M}{\partial_{\x}\mathcal{L}_c(\x, \y_c)} = \sum_{c=1}^{M} \bf{r}_c
\]
As we can see, the joint gradient is the sum of gradients from each individual task. Then we consider the expectation of the square of the $L_2$ norm of the joint gradient:

\[\mathbb{E}(\lVert \R \rVert_2^2) = \mathbb{E}\left(\lVert \frac{1}{M} \sum_{c=1}^M{\bf{r}_c}\rVert_2^2\right) = \frac{1}{M^2}\mathbb{E}\left(\sum_{c=1}^M{\lVert\bf{r}_c\rVert^2} + 2 \sum_{i=1}^M{\sum_{j=1}^{i-1}{\bf{r}_i \bf{r}_j}}\right)
\]

\[ \mathbb{E}(\Vert \R \rVert_2^2) = \frac{1}{M^2}\left(\sum_{i=1}^M\mathbb{E}{\lVert \bf{r}_i \rVert^2} + 2 \sum_{i=1}^M{\sum_{j=1}^i\mathbb{E}({\bf{r}_i \bf{r}_j})}\right)
\]

Since \[
\mathrm{Cov}(\bf{r}_i, \bf{r}_j) = \mathbb{E}(\bf{r}_i \bf{r}_j) - \mathbb{E}(\bf{r}_i)\mathbb{E}(\bf{r}_j)
\]
According to the assumption \[
\mathbb{E}(\bf{r}_j) = \0
\]

We know 
\[
\mathrm{Cov}(\bf{r}_i, \bf{r}_j) = \mathbb{E}(\bf{r}_i \bf{r}_j)
\]

Then we get 
\[
\mathbb{E}(\lVert \R \rVert_2^2) = \frac{1}{M^2}\left(\sum_{i=1}^M\mathbb{E}({\mathrm{Cov}(\bf{r}_i, \bf{r}_i)}\right) + 2 \sum_{i=1}^M{\sum_{j=1}^i\mathbb{E}({\mathrm{Cov}(\bf{r}_i, \bf{r}_j)})} = \frac{1}{M^2}\left(\sum_{i=1}^M \sigma^2 + 2 \sum_{i=1}^M{\sum_{j=1}^i\mathbb{E}[{\mathrm{Cov}(\bf{r}_i, \bf{r}_j)}]}\right)
\]

\noindent where $\sigma^2 = \mathrm{Cov}(\bf{r}_i, \bf{r}_i)$

% \[
%     % \hspace{-60}
%     &\mathbb{E}(||\R||_2^2) = \mathbb{E}(||\frac{1}{M} \sum_{i=1}^M{\r_i}||_2^2)\\
%     &= \frac{1}{M^2}\mathbb{E}(\sum_{i=1}^M{||\r_i||^2} + 2 \sum_{i=1}^M{\sum_{j=1}^i{\r_i \r_j}}) \\
%     &= \frac{1}{M^2}(\sum_{i=1}^M\mathbb{E}{||\r_i||^2} + 2 \sum_{i=1}^M{\sum_{j=1}^i\mathbb{E}({\r_i \r_j})}) \\
%     &= \frac{1}{M^2}(\sum_{i=1}^M\mathbb{E}({\mathrm{Cov}(\r_i, \r_i)}) + 2 \sum_{i=1}^M{\sum_{j=1}^i\mathbb{E}({\mathrm{Cov}(\r_i, \r_j)})}) 
% \]
% \noindent The last equation holds due to our assumption that the mean value of gradient is $\0$.

Thus, the adversarial vulnerability is:

\[
\mathbb{E}_{\x}[\Delta \mathcal{L}]  \propto \mathbb{E}_{\x}\left[\lVert \partial_{\x} \mathcal{L}_{all}(\x, \overline{\y})\rVert_2 \right] 
= \frac{\sqrt{(1 + \frac{2}{M} \sum_{i=1}^{M} \sum_{j=1}^{i-1} \frac{\mathrm{Cov}(\bf{r}_i, \bf{r}_j)}{\mathrm{Cov}(\bf{r}_i, \bf{r}_i)})}}{\sqrt{M}}
\]

\end{proof}{}

For a special case where all the tasks are independent of each other, independent gradients with respect to the input are produced, we have the following corollary:

\begin{corollary}\label{th:1}
\textbf{(Adversarial Vulnerability of Model for Multiple Independent Tasks)}
 If the output tasks selected are independent of each other, and the gradient for each task is i.i.d. with zero mean, then the adversarial vulnerability of given model is proportional to $\frac{1}{\sqrt{M}}$, where $M$ is the number of independent output tasks selected. 
\end{corollary}

\begin{proof}
According to the independent assumption, we have 
\[
\mathrm{Cov}(\bf{r}_i, \bf{r}_j)  = 0
\]

Let $\sigma^2 = \mathrm{Cov}(\bf{r}_i, \bf{r}_i) $. Thus we get the adversarial vulnerability to be:

\[
\mathbb{E}_{\x}[\Delta \mathcal{L}]  \propto \mathbb{E}_{\x}\left[||\partial_{\x} \mathcal{L}_{all}(\x, \overline{\y})||_2 \right] 
= \sqrt{ \frac{\sigma^2}{M}} \propto \frac{1}{\sqrt{M}}
\]
\end{proof}{}

\section{Experimental Setup}

\subsection{\cityscape}
We train DRN-105 model and evaluate against multi-task attack. We follow the original architecture setup of the original DRN paper \cite{Yu2017}.  We used 93 layers in the shared backbone encoder network, and 13 layers in the decoder branch for individual task prediction. We use a batch size of 24. We start with a learning rate of 0.01 and decrease the learning rate by a factor of 10 after every 100 epochs. We trained the model for 250 epochs.
 
 We train multi-task model against single task attack using DRN-22 model. We use 18 layers in the shared backbone encoder network, and 9 layers in the decoder branch for individual task prediction. We use batch size of 32. We optimize with SGD, with learning rate of 0.01, then decrease it to 0.001 at 180 epoch. We train model for 200 epoch in total.  We applied a weight decay of 0.0001 for all the models.

\subsection{\taskonomy}

~~~~~ Taskonomy dataset \cite{taskonomy} consists of millions of indoor scenes with labels for multiple tasks, we use 11 tasks including semantic segmentation, depth estimation, 2D and 3D edge detection, normal vector estimation, reshading, 2D and 3D keypoint detection, Euclidean depth, auto-encoding, and principal curvature estimation. We use the publicly available Tiny version of dataset, which consists of 9464 images from 1500 rooms. We use examples from 80\% of the rooms as training data and examples from 20\% of the rooms as test data. Images from the same room are only contained in either the training set or the test set, and not in both. The quality of the model is measured by its ability to generalize to new rooms.

For learning a multi-task model for joint robustness, we follow the set up described in \cite{multi-task-onomy}. We train a ResNet-18 as the shared backbone encoder network for all the tasks. Each multi-task model consists of 1 to 6 different tasks. We use an input size of $512 \times 512$. We use an 8 layer decoder for each individual task prediction. Following the data preprocessing of \cite{multi-task-onomy}, we apply equal weights to all the tasks. Start from task "semantic segmentation" (s), we add tasks "depth" (d), "edge texture" (e), "keypoints 2d" (k), "normal" (n), and "reshading" (r). Thus we train 6 models `s,' `sd,' `sde,' `sdek,' `sdekn,' `sdeknr.' We also train `d,' `e,' `er,' `k,' `ks,' `ksd' tasks, so that we can analysis the trend of 4 tasks' performance after multitask learning. We use the same learning rate schedule for all the models --- SGD with learning rate 0.01 and momentum 0.99. We decrease the learning rate at 100 epoch by 10 times. We train all the models for 150 epoch. Results are shown in Figure 5 in the main paper.

\begin{table*}
\scriptsize
    \centering
    \begin{tabular}{c|c|cccccccccccc}
         \toprule
         \midrule
         \multicolumn{13}{c}{PGD Adversarial} \\
        %  \multicolumn{10}{c}{\textbf{Cifar10 (WRN)}}
        %  \midrule
        %  & &  \multicolumn{10}{c}{ \textbf{Auxiliary Tasks}}  \\
           & Baseline & SemSeg & DepthZ & Edge2D & Normal & Reshad & Key2D & Key3D & DepthE & AutoE & Edge3D & PCurve\\
        %   & & - & - & - & 1 & 7 & 30 & 20 & 40 \\
         \midrule  
        %  \parbox[t]{3mm}{\multirow{4}{*}{\rotatebox[origin=c]{90}{Methods}}}
          Semseg * & {13.360} & --- & \fbox{\textbf{19.320}} & \textbf{13.950} & \textbf{16.630} & \textbf{14.580} & \textbf{15.700} & \textbf{13.780} & \textbf{14.910} & \textbf{14.110} & \textbf{14.720} & \textbf{14.900} \\
         
          DepthZ $(10^{-2})$ & {11.491} & \fbox{\textbf{4.712}} & --- & \textbf{6.780} & {11.617} & {12.412} & \textbf{11.120} & \textbf{8.358} & \textbf{10.498} & \textbf{4.981} & {12.230} & \textbf{5.035} \\

         Edge2D $(10 ^ {-2})$ & {10.672} & \textbf{9.841} & \fbox{\textbf{9.363}} & --- & \textbf{9.546} & \textbf{9.943} & \textbf{9.732} & \textbf{9.654} & \textbf{9.714} & \textbf{9.941} & \textbf{9.978} & \textbf{10.095} \\

         Normal $(10 ^ {-2})$ & {40.926} & \fbox{\textbf{35.171}} & {42.871} & \textbf{39.335} & --- & \textbf{40.501} & \textbf{39.462} & \textbf{39.930} & {42.071} & \textbf{35.726} & \textbf{37.070} & {41.212} \\

        Reshad $(10 ^ {-2})$ & {57.900} & \fbox{\textbf{48.800}} & \textbf{57.800} & \textbf{55.000} & \textbf{56.500} & --- & \textbf{55.900} & \textbf{53.300} & {60.000} & {61.000} & \textbf{49.300} & \textbf{57.600} \\

        Key2D $(10 ^ {-2})$ & {11.700} & \textbf{10.900} & \textbf{10.900} & \textbf{10.700} & \textbf{10.500} & \textbf{10.900} & --- & \textbf{11.000} & \textbf{10.600} & \textbf{11.000} & \textbf{10.800} & \fbox{\textbf{10.600}} \\

        Key3D $(10 ^ {-2})$ & {49.700} & \fbox{\textbf{31.000}} & \textbf{49.600} & \textbf{50.800} & \textbf{45.900} & \textbf{42.200} & \textbf{43.800} & --- & {51.200} & \textbf{32.600} & {53.400} & {52.900} \\

         DepthE $(10^{-3})$ & {4.850} & \textbf{3.530} & \textbf{3.390} & \textbf{3.250} & \textbf{4.270} & {5.670} & \textbf{3.670} & \textbf{3.730} & --- & \textbf{3.700} & \textbf{3.330} & \fbox{\textbf{2.930}} \\

         AutoE $(10^{-2})$ & {59.300} & \fbox{\textbf{57.800}} & {60.300} & \textbf{58.300} & {62.300} & {59.400} & {59.300} & {60.700} & \textbf{58.200} & --- & {60.500} & {61.500} \\

         Edge3D $(10^{-2})$ & {15.900} & \textbf{14.600} & \textbf{15.300} & {16.300} & \textbf{15.400} & \textbf{15.200} & \textbf{15.600} & {16.900} & \textbf{15.400} & \fbox{\textbf{12.600}} & --- & \textbf{14.800} \\

         PCurve $(10^{-4})$ & {11.500} & \textbf{8.920} & \fbox{\textbf{8.900}} & \textbf{10.400} & \textbf{9.230} & \textbf{9.620} & \fbox{\textbf{8.900}} & \textbf{10.400} & \textbf{11.100} & \textbf{9.190} & \textbf{10.400} & --- \\
         
         \toprule
         \midrule
         \multicolumn{13}{c}{Clean} \\
         \midrule
         SemSeg * & {43.190} & {---} & \textbf{46.300} & \textbf{46.180} & \fbox{\textbf{46.350}} & \textbf{46.240} & \textbf{45.440} & \textbf{45.620} & \textbf{44.690} & \textbf{44.500} & \textbf{45.320} & \textbf{44.490} \\

         DepthZ  $(10^{-2})$ & {2.852} & \fbox{\textbf{2.734}} & {---} & {3.880} & \textbf{2.846} & \textbf{2.505} & {2.874} & {3.562} & {3.339} & {3.171} & {3.088} & {4.690} \\

         Edge2D $(10^{-2})$ & {3.384} & {3.922} & \textbf{3.382} & {---} & {3.507} & {3.435} & \textbf{3.330} & {3.522} & {3.433} & {3.574} & {3.569} & {3.454} \\

         Normal $(10^{-2})$ & {6.997} & {7.181} & {7.093} & {7.006} & {---} & \textbf{6.989} & \textbf{6.990} & {7.182} & \textbf{6.940} & \fbox{\textbf{6.864}} & \textbf{6.931} & {7.141} \\

         Reshad $(10^{-2})$  & {8.027} & \textbf{7.985} & {8.103} & \textbf{7.941} & \textbf{7.901} & {---} & {8.041} & \textbf{7.957} & \textbf{7.940} & \fbox{\textbf{7.890}} & {8.065} & {8.150} \\

         Key2D $(10^{-2})$ & {4.156} & \textbf{4.116} & \textbf{3.897} & \fbox{\textbf{3.795}} & \textbf{3.865} & \textbf{4.147} & {---} & \textbf{3.944} & \textbf{3.857} & \textbf{3.823} & \textbf{3.850} & \textbf{3.878} \\

         Key3D $(10^{-2})$ & {8.771} & \textbf{8.445} & \textbf{8.686} & \textbf{8.514} & \textbf{8.610} & \fbox{\textbf{8.318}} & \textbf{8.703} & {---} & \textbf{8.492} & \textbf{8.366} & \textbf{8.362} & \textbf{8.578} \\

         DepthE $(10^{-3})$ & {6.373} & {6.575} & \textbf{5.946} & \textbf{6.350} & \textbf{6.236} & \textbf{5.802} & {6.418} & {6.470} & {---} & \textbf{5.948} & \fbox{\textbf{5.715}} & \textbf{6.251} \\

         AutoE $(10^{-2})$  & {3.470} & {3.616} & {3.709} & {3.548} & {3.587} & {3.540} & {3.780} & {3.761} & {3.542} & {---} & {3.530} & {3.553} \\
         
         Edge3D $(10^{-2})$  & {4.649} & {4.695} & \textbf{4.608} & {4.727} & \textbf{4.562} & {4.725} & \textbf{4.364} & \textbf{4.635} & \textbf{4.611} & \textbf{4.210} & {---} & \fbox{\textbf{3.703}} \\

         PCurve $(10^{-4})$ & {8.017} & {8.360} & {8.184} & {8.353} & {8.541} & \fbox{\textbf{7.232}} & \textbf{7.733} & \textbf{7.725} & {8.153} & \textbf{7.854} & \textbf{7.732} & {---} \\

         \bottomrule
    \end{tabular}
\caption{ The absolute performance of all models trained on two tasks (Relative are shown in Figure 7 in the main paper).  Each row in the first column lists the name of the main task. The second column (baseline) shows the performance of a model trained on a single task. The * in the row indicates the mIoU score for semantic segmentation, for which higher is better. The values in the other rows of the table show the l1 loss, for which lower is better. The $(10^{-n})$ in the first column indicates the unit for the error. `SemSeg' denotes `semantic segmentation,' `DepthZ' denotes `depth estimation,' `Edge2D' denotes `2D edge detection,' `Normal' denotes `Normal Vector estimation', `Reshad' denotes `Reshading,' `Key2D' denotes `2D Keypoint detection,' `Key3D' denotes `3D Keypoint detection,' `DepthE' denotes `Euclidean depth,' `AutoE' denotes `Auto Encoder,' `Edge3D' denotes `3D Edge detection,' `PCurve' denotes `Curvature estimation.'
Values superior to the baseline are \textbf{bold}, and the best performance for each row is in a \textbf{box}. The table lists the IoU (large is better) for the segmentation model, and error (small is better) for all the other tasks. We pair each selected model with 11 other models. All the models converge after training for 150 epochs. Overall, training on two tasks can help the individual task's adversarial robustness on \textbf{90/110} cases, while surpassing the baseline's performance on the clean examples on \textbf{70/110}. For instance, the adversarial robustness for the semantic segmentation and keypoints3D estimation is always improved by multi-task learning while the clean accuracy also improves. The results on 11 tasks support our claim that training on multiple tasks improves adversarial robustness.
} \label{seg:all_11_aux}
\end{table*}

\begin{table*}
\scriptsize
    \centering
    \begin{tabular}{c|cccccccccccc}
         \toprule
         \midrule
         \multicolumn{12}{c}{$\lambda_a$} \\
        %  \multicolumn{10}{c}{\textbf{Cifar10 (WRN)}}
        %  \midrule
        %  & &  \multicolumn{10}{c}{ \textbf{Auxiliary Tasks}}  \\
            & SemSeg & DepthZ & Edge2D & Normal & Reshad & Key2D & Key3D & DepthE & AutoE & Edge3D & PCurve\\
        %   & & - & - & - & 1 & 7 & 30 & 20 & 40 \\
         \midrule  
        %  \parbox[t]{3mm}{\multirow{4}{*}{\rotatebox[origin=c]{90}{Methods}}}
          Semseg * & 0  &  0.01 & 9.01 & 0.1 & 0.1 & 0.01 & 0.01 & 0.1 & 0.01 & 0.1 & 0.01\\
         
          DepthZ  & 0.1 & 0 & 0.1 & 0.01 & 0.1 & 0.1 & 0.01 & 0.1 & 0.1 & 0.1 & 0.01  \\

         Edge2D  & 0.1 & 0.1 & 0 & 0.1 & 0.1 & 0.01 & 0.1 & 0.1 & 0.01 & 0.01 & 0.1 \\

         Normal & 0.1 & 0.01 & 0.1& 0 & 0.01& 0.1& 0.01 &0.1& 0.1& 0.01 &0.01  \\

        Reshad  & 0.01 &0.1 &0.01& 0.01& 0 &0.1 &0.01& 0.01 &0.1& 0.01& 0.1\\

        Key2D  & 0.1& 0.1& 0.01 &0.01& 0.01& 0& 0.01& 0.01 &0.01 &0.01& 0.1\\

        Key3D & 0.1 &0.01& 0.1 &0.1 &0.1& 0.1& 0& 0.1& 0.1& 0.01& 0.01\\

         DepthE & 0.1& 0.01&  0.01& 0.01 &0.01 &0.1 &0.01 &0& 0.01 &0.1& 0.1 \\

         AutoE  & 0.1 &0.01& 0.01 &0.1& 0.01& 0.1& 0.01 &0.1& 0 &0.01 &0.1 \\

         Edge3D & 0.1& 0.01& 0.01& 0.01 &0.01& 0.1& 0.01& 0.01& 0.1 &0 &0.1 \\

         PCurve & 0.1& 0.01& 0.1& 0.01 &0.01& 0.1& 0.01& 0.1& 0.01& 0.01& 0\\

         \bottomrule
    \end{tabular}
\caption{The $\lambda_a$ value for the auxiliary task for Figure 7 in the main paper.} \label{seg:all_11_aux_lam}
\end{table*}

% follows the same WRN model as Madry et al \cite{madry} across all our models, as shown in Table \ref{tab:wrn-archi}. Also, we adopt the same SGD optimization method with the same learning rate decay strategy as Madry's, where we start with learning rate of 0.1 and decrease it to 0.01 at 50k iterations. We run it for 55k iterations before stopping. We train all the models with a batch size of 50. We implement the ALP on CIFAR-10 because it is not implemented in the original ALP paper, where we do improve the adversarial accuracy significantly. To achieve a fair comparison, we all follow the hyper-parameters set-up in \cite{madry}. It took a day and a half before our training converges. We set the $\lambda=0.5$ for ALP and do not use label smoothing. For TLA method, we adopt $\lambda_1 = 2$, $\lambda_2=0.001$, margin $\alpha=0.03$, mini-batch size for the negative sample selection as 500. Our TLA improves the robust accuracy over ALP baseline for \textbf{4.12\%} and AT baseline for \textbf{4.82\%}.

For training robust models on select tasks, we use ResNet-18 as the shared encoder network. We select 11 tasks trained in pairs with each other, which results in 110 models. We study their robustness under a single-task attack. We follow the data processing in \cite{taskonomy}.  For each task we considered, we try weights of 0.1 and 0.01 for the auxiliary task, and choose the weight that produces higher robust accuracy. The chosen $\lambda_a$ for the auxiliary tasks are shown in Table \ref{seg:all_11_aux_lam}. The selection of weights is important due to the complex interactions of different tasks \cite{MTLworkshop}. We follow the setup in \cite{taskonomy}, and subsample the image from 512 to 256 using linear interpolation. For segmentation, reshading, keypoint 3D, depth Euclidean, Auto Encoder, principle curvature, we use SGD, with learning rate 0.01 and decrease by 10 times at 140 epoch. For the other tasks we use adam, with learning rate 0.001 and decrease by 10 times at 120 and 140 epoch. Due to the inherent difference between different tasks, we use different optimizer for different tasks for better convergence. All the models are trained for 150 epoch. All the results are shown in Table \ref{seg:all_11_aux}. As we can see, learning versatile, multi-task models improves adversarial robustness on 90/110 tasks.

\subsection{Adversarial Training}

We present the details for multi-task adversarial training in Algorithm \ref{alg:1}.  For single task model, we choose $S=\{\{T_m\}\}$. The algorithm is the same as the adversarial training procedure of Madry et. al. \cite{madry}. For multi-task model, we set $S = \{\{T_m\}, \{T_m, T^{(1)}_a, ..., \}\}$, thus the generated adversarial images under multi-task are more diversified compared with single-task models. In addition, all the adversarial examples are trained on multi-task loss function, where the auxiliary task can introduce useful knowledge for learning the robust main task. We use $\lambda=0.01$ for the auxiliary task. For all the task, we train using SGD optimizer with batch size of 32, for 200 epoch. We start with learning rate of 0.01, and decrease the learning rate by 10 times at 180 epoch. The experiment are conducted on Cityscapes dataset.

\begin{algorithm}[t]
\caption{Adversarial Training with Multi-task Learning}
\label{alg:1}
%{\bf Input:} .
%{\bf Output:} 
{\bf Input:} Initialized networks $F_i$, dataset $D$, main task $T_m$, auxiliary task $T^{(i)}_a$. Construct multi-task combination set $S = \{\{T_m\}, \{T_m, T^{(1)}_a, ..., \}\}$

{\bf Output:} 

\begin{algorithmic}
\FOR{number of training epochs}
    \FOR{number of iterations in each epoch}
        \STATE Sample minibatch of $n$ images $\x$ from $D$.
        \FOR{each task combination $S_t$ in $S$}
        
            \STATE Let $\mathcal{L}_t(\x,\y) = \sum_{i} \lambda_i \ell(F_i(\x), \y_i)$,  where $i = 1, ..., \#(S_t)$,  $T_i \in S_t$.
            
            \STATE Compute adversarial attack images $\x_{adv}$
            \begin{equation*}
            \begin{aligned}
            \argmax_{\x_{adv}} \mathcal{L}_t(\x_{adv}, \y),    \text{s.t.} ||\x_{adv} - \x||_p \leq r
            \end{aligned}
            \end{equation*}
            
            \STATE Training the multi-task model using the generated attack image $\x_{adv}$ by optimizing the following loss function:
            \begin{equation*}
            \begin{aligned}
            \min{\mathcal{L}_t(\x, \y)}
            \end{aligned}
            \end{equation*}

        \ENDFOR

    \ENDFOR
\ENDFOR
\RETURN Neural network model $F_i$

\end{algorithmic}
\end{algorithm}

\end{subappendices}

\end{document}